\documentclass{article}
\usepackage[a4paper]{geometry}

\usepackage[utf8]{inputenc} 
\usepackage[T1]{fontenc}    
\usepackage{hyperref}       
\usepackage{url}            
\usepackage{booktabs}       
\usepackage{amsfonts}       
\usepackage{nicefrac}       
\usepackage{microtype}      
\usepackage{algorithmic}
\usepackage{amsmath}
\usepackage{amsthm}
\usepackage{amssymb}
\usepackage{enumerate}
\usepackage{amssymb}
\usepackage{amssymb}
\usepackage{enumerate}
\usepackage{paralist}
\usepackage[]{algorithm2e}
\usepackage{graphicx}
\usepackage{caption}
\usepackage{subcaption}

\newcommand{\cH}{\mathcal{H}}
\newcommand{\cU}{\mathcal{U}}

\newcommand{\cX}{\mathcal{X}}
\newcommand{\cD}{\mathcal{D}}
\newcommand{\E}{\mathbb{E}}
\newcommand{\close}{\epsilon}
\newcommand{\regular}{\alpha}

\newcommand{\poly}{\mathrm{poly}}
\newcommand\VC[1]{\mathrm{VCdim}(#1)}

\newtheorem{lemma}{Lemma}
\newtheorem{theorem}[lemma]{Theorem}
\newtheorem{corollary}[lemma]{Corollary}
\newtheorem{definition}[lemma]{Definition}

\newtheorem{claim}[lemma]{Claim}

\title{A General Memory-Bounded Learning Algorithm}
\date{}

\author{%
  Michal ~Moshkovitz
  \\
  Edmond and Lily Safra Center for Brain Sciences\\
  The Hebrew University\\
  Jerusalem 91904, Israel \\
  \texttt{michal.moshkovitz@mail.huji.ac.il}\\ 
   \and
   Naftali Tishby \\
  The Rachel and Selim Benin School of Computer Science and Engineering \\
  The Hebrew University\\
  Jerusalem 91904, Israel \\
   \texttt{tishby@cs.huji.ac.il} \\
}

\begin{document}

\maketitle

\begin{abstract}
Designing bounded-memory algorithms is becoming increasingly important nowadays. 
Previous works studying bounded-memory algorithms focused on proving impossibility results, while the design of bounded-memory algorithms was left relatively unexplored.
To remedy this situation, in this work we design a general bounded-memory learning algorithm, when the underlying distribution is known. 
The core idea of the algorithm is not to save the exact example received, but only a few important bits that give sufficient information. 
This algorithm applies to any hypothesis class that has an ``anti-mixing'' property.
This paper complements previous works on unlearnability with bounded memory and provides a step towards a full characterization of bounded-memory learning. 
\end{abstract}
\section{Introduction}
The design of learning algorithms that require a limited amount of memory is more crucial than ever as we are amidst the big data era. 
An enormous amount of new data is created worldwide every second \cite{InternetLiveStats}, while learning is often performed on low memory devices (e.g., mobile devices). 
To bridge this gap, low memory algorithms are desired.
Moreover, bounded-memory learning has connections to other fields 
as artificial and biological neural networks can be viewed as bounded-memory algorithms (see \cite{moshkovitz17}).
Prior to this paper there were many works that provided lower bounds for bounded-memory learning \cite{shamir14, steinhardt16, raz16, kol17, moshkovitz17a, moshkovitz17b, raz17, garg17, beame17}. 
  Specifically, \cite{moshkovitz17b}  defined a ``mixing'' property for classes which is a combinatorial condition that if satisfied the class cannot be learned with bounded-memory. 
  There are also a few upper bounds, but for specific classes (e.g., 
\cite{rosenblatt57}).

 In this paper we define a combinatorial property, \emph{separability}, which is basically ``anti-mixing''.
 We prove that if a class is separable, then it can be properly\footnote{A proper learner for a class $\cH$ always returns hypothesis $h\in\cH$ in the class.} learned under a known 
 distribution with a bounded-memory algorithm.
This property use the same terms as mixing (more details appear in Section~\ref{sec:related_work}), thus, it is an important step towards bounded-memory characterization.
We also provide a general bounded memory algorithm, and prove that this algorithm function correctly for any separable class.
Interestingly, we show that this algorithm can be implemented also in the statistical query model of \cite{kearns98}, and hence is robust to noise. 
Finally, we exemplify the algorithm on several natural classes 
 and show an implementation that is both time and memory efficient.





\subsection{The General Algorithm: a few Examples}
To show the generality of the algorithm proposed in this paper we show how to apply it to several natural classes. 
These classes have varied structures: different dimension and  different VC-dimension values. Nevertheless, we prove that all these classes satisfy one unified general condition, separability.
We also present efficient implementations of the algorithm both in time and space for these classes. 

\textbf{Decision Lists.} 
A decision list is a function defined over $n$ Boolean inputs of the following form:
\[\textbf{if } \ell_1 \textbf{ then } b_1 \textbf{ else if } \ell_2 \textbf{ then } \ldots \textbf{if }\ell_k \textbf{ then } b_k \textbf{ else } b_{k+1},\]
where $\ell_1,\ldots,\ell_k$ are literals over the $n$ Boolean variables and $b_1,\ldots,b_{k+1}$ are bits in $\{0,1\}.$
This class was introduced by 
\cite{rivest87} and it has interesting relationships with important classes such as threshold functions,  2-monotonic functions, read-once
functions, and more \cite{eiter02}.
There are several works on learning decision lists \cite{nevo02, dhagat94, klivans06}, however those works learn this class under several assumptions. Ignoring those assumptions leaves their algorithm with a non-polynomial number of examples, which is  a minimal requirement  for any learning algorithm.

\textbf{Equal-Piece Classifiers.} 
The domain is a discretization of the segment $[0,1]$. Each hypothesis corresponds to a few disjoint segments of size $p$ and an example gets the value $1$ if it is inside one of the segments and $0$ otherwise.
This class is an exemplar class for weak-learnability \cite{shalev14}.


\textbf{Discrete Threshold Functions.} 
As a sanity check, we also apply the algorithm to the discrete threshold functions, where the domain is the discretization of the segment $[0,1]$ and each hypothesis corresponds to a number $\theta\in[0,1]$. An example is labeled by $1$ if it is smaller than $\theta$ and $0$ otherwise. 
There is a simple learning algorithm for this class: save in memory the largest example with label $1.$ 
We show that indeed this class can be learned using our general algorithm. 
This class is similar to equal-piece classifier with $p=1.$

\subsection{Intuition for the General Algorithm}
The general bounded-memory algorithm saves in memory a subset of hypotheses $T\subseteq\cH$ that contains the correct hypothesis $f$, with high probability.
At each iteration, it draws a few examples and then removes hypotheses from $T$. 
Crucially, if $\cH$ is separable, the algorithm removes a \emph{large} fraction of hypotheses from $T$, and this is why it stops after a small number of iterations. 
Importantly, the removal is done while saving only a few bits of memory. 
This is done using ideas from graph theory, inspired by \cite{moshkovitz17b}, as explained next.


A hypothesis class $\cH=\{h:\cX\rightarrow\{0,1\}\}$ over domain $\cX$ can be represented as a bipartite graph in the following way. 
The vertices are the examples $\cX$ and the hypotheses $\cH$, and the edges connect every example $x\in\cX$ to a hypothesis $h\in \cH$ if and only if $h(x)=1.$ 
The density $d(S,T)$ between two subset of vertices $S$ and $T$ is the fraction of edges between them.

The key idea of the general learning algorithm is to estimate the density $d(S,f)$, where $f$ is the correct hypothesis and $S\subseteq\cX$ is a set of examples with heavy weight according to the known distribution.
This density can be estimated with a few bits since it can be written as an expectation over examples in $S$, 
and most of the received examples  are from $S$ as it is heavy-weight. 
Using this estimation, the algorithm rules out from $T$ any hypothesis $h\in T$ with $d(S,h) \not\approx d(S,f)$.
The separability property ensures that at each step the algorithm rules out many hypotheses.  

\subsection{Informal Summary of our Results}
The results are informally summarized  below.
\begin{enumerate}
    \item We introduce the combinatorial condition of \emph{separability} for hypothesis classes, which is closely related to anti-mixing (the connection between the two definitions is discussed in Section~\ref{sec:related_work}). 
    \item We present a general memory-bounded proper learning algorithm in the case where the examples are sampled from a known distribution. 
    We prove the correctness of this algorithm in the case where the classes satisfy the separability condition. 
    \item We prove that the general algorithm is also a statistical query algorithm.
    \item We exemplify our algorithm on several natural algorithms: decision lists, equal-piece classifiers, and discrete threshold functions.
\\
\end{enumerate}

\subsection{Paper Outline} 
Related work is discussed in Section~\ref{sec:related_work}.
 In Section~\ref{sec:algorithms} we present the general bounded-memory algorithm and prove that it is also a statistical query algorithm. 
 In Section~\ref{sec:definitions} we formally present the notion of separability. 
 In Section~\ref{sec:applications} we show that this algorithm can be used to properly learn the three classes presented above with bounded memory.  
 In this section we also present the time and memory efficient algorithm for the class of decision lists.  
 

\section{Related Work}\label{sec:related_work}
The fundamental theorem of statistical learning gives an exact combinatorial characterization of learning classification problems \cite{shalev14}. 
The theorem also provides an inefficient learning rule, empirical risk minimization, that can be used to learn any learnable class. 
This paper is a step towards a ``fundamental theorem of \emph{bounded-memory} learning'' as it (i) gives a combinatorial condition, separability, for bounded-memory learnability that is similar to the mixing condition for unlearnability with bounded-memory (ii) suggests a general algorithm for bounded-memory learning any separable class. 

Many works \cite{shamir14, steinhardt16, raz16, kol17, moshkovitz17a, moshkovitz17b, raz17, garg17, beame17} have discussed the limitations of bounded-memory learning. 
The work \cite{moshkovitz17b}  defined a ``mixing'' property for classes, that if satisfied the class cannot be learned with bounded-memory. 
 Colloquially, mixing states that \emph{for any} subset of hypotheses $T$ and \emph{for any} subset of examples $S$, the number of edges between $S$ and $T$, $E(S,T)$, is as expected (about $p|S||T|$, where $p$ is the density of the graph). On the other hand, anti-mixing states that \emph{for any} $T$, \emph{there is} $S$ such that $E(S,T)$ is far from what we expect. For comparison, the negation of mixing means that \emph{there is} $T$ and \emph{there is} $S$ such that $E(S,T)$ is far from what we except. Thus, the negation of mixing and anti-mixing are very similar definitions but not exactly the same. 
It is an important open problem to provide a full characterization of bounded-memory learning. 



The statistical query (SQ) model  was introduced by 
\cite{kearns98} to  provide a general framework for learning in the presence of classification noise. 
We prove that our algorithm works in the statistical query model of Kearns and is thus robust to classification noise.
%
A characterization of learnability with statistical queries was first given by \cite{blum94} where the SQ-dimension was introduced. 
This is an \emph{exact} measure for weak learnability since sq-dim$(\cH)=d$ if and only if the class $\cH$ can be learned (weakly) with $poly(d)$ statistical queries. Given the connections between statistical queries and bounded-memory algorithms \cite{steinhardt16, feldman16} one might hope that sq-dim fully captures bounded memory too. 
This, however, is not known to be true. 
This is why other approaches to bounded-memory characterization are needed. 
This paper provides such a promising approach as it uses the same terms as mixing and has a striking similarity to non-mixing. 

%
%

On the surface it seems that memory-bounded learning algorithms can be obtained from algorithms that compress the labeled examples to fit in a small space (Occam's Razor paradigm by \cite{blumer1987occam}, sample compression learning algorithms \cite{littlestone86,floyd89}). However, this is not the case, since compression algorithms work in an offline model in which all the labeled examples are stored, but their storage does not count towards the memory usage of the learning algorithm. In the bounded-memory model, on the other hand, the examples are received in an online fashion and storing them counts against the memory bound of the algorithm.

\section{A General Bounded Memory Algorithm}\label{sec:algorithms}
Let $\cH$ be hypothesis class over a domain set $\cX$.
A learning algorithm receives, in an on-line fashion, a series of labeled examples $(x,f(x))$, where $f$ is the true hypothesis.
The goal of the algorithm is to return a hypothesis $h\in\cH$ that is close to $f.$ 
Any bounded-memory algorithm can be described using a branching program (see Figure~\ref{fig:branching_program}), which is a layered graph. Each layer corresponds to a time step, i.e., number of examples received so far. Each layer contains all the possible memory states. If the algorithm uses $b$ bits, then there are $2^b$ memory states.  
One can learn any class $\cH$ over domain $\cX$ with $O(\log|\cH|)$ examples and $O(\log|\cX|\log|\cH|)$ memory bits. 
 Thus, a learning algorithm that uses much less memory bits, $o(\log|\cX|\cdot\log|\cH|)$ bits, is called a \emph{bounded-memory} algorithm.

One of the main contribution of this work is designing a general learning algorithm that uses a bounded amount of bits. 
This 
algorithm saves in memory a subset of hypotheses $T$ that, with high probability, contains the correct hypothesis $f$. 
At the beginning of the algorithm, $T$ contains all the hypotheses in $\cH$. 
At each step the algorithm reduces a large fraction of $T$ while only using a few bits. 
The algorithm acts differently depending on whether $T$ contains a large subset of hypotheses that are close to each other or not. 
We call the former case \emph{tightness}, as formalized next. 

Two hypotheses $h_1,h_2\in\cH$ are \emph{$\close$-close} if $\Pr_{x\sim \cD}(h_1(x)\neq h_2(x))\leq\epsilon,$ where $\cD$ is the known distribution over the examples\footnote{For ease of presentation, in the rest of the paper we focus on the case that $\cD=\cU$ is the uniform distribution.}.
An \emph{$\close$-ball} with center $h\in \cH$ is the set $B_h(\close) = \{h'\in \cH \,\vert\, h' \text{ and } h \text{ are } \close\text{-close}\}.$ 
A subset $T\subseteq\cH$ is $(\regular,\close)$-\emph{tight} if there is a hypothesis $h$ with $|T\cap B_h(\close)|\geq \regular|T|$.
Note that $\regular$ and $\close$ are related as large $\close$ implies that $\regular$ is also large.
At each step, the algorithm distinguishes between the cases that $T$ is tight and $T$ is not tight and handle each case separately.
  
\textbf{$T$ is $(\regular,\close)$-\emph{tight}:} in this case there is a hypothesis $h$ with $|T\cap B_h(\close)|\geq \regular|T|$. 
The algorithm tests if $h$ is $\epsilon$-close to the correct hypothesis $f$. 
This is done using a few random examples, as described in Algorithm~\ref{alg:is_close}.  
If $h$ is $\close$-close, then the algorithm halts. Otherwise, the algorithm can safely delete $B_h(\close)$ from $T$, i.e., 
in this case the algorithm can reduce many, $\regular|T|$, hypotheses from $T$.  

\textbf{$T$ is not $(\regular,\close)$-\emph{tight}:} for this case we use ideas from graph theory. 
A hypothesis class $\cH$ over domain $\cX$ can be represented as a bipartite-graph in the following way. 
The vertices are the hypotheses $\cH$ and the examples $\cX$, and the edges connect every hypothesis $h\in \cH$ to an example $x\in\cX$ if and only if $h(x)=1.$ 
We call the appropriate bipartite graph the \emph{hypotheses graph} of $\cH.$
For any graph $(A,B,E)$, the \emph{density} between sets of vertices $S \subseteq A$ and $T \subseteq B$ is $d(S,T)=\frac{e(S,T)}{|S||T|}$, where $e(S,T)$ is the number of edges with a vertex in $S$ and a vertex in $T$. In case that $T$ contains only one vertex $T=\{v\}$, we simply write $d(S,v)$. 

For any heavy-weight $S$ (i.e., $|S|\geq\alpha|\cX|$ under the uniform distribution) the density $d(S,f)$ between $S$ and the correct hypothesis $f$ can be easily estimated without saving many bits in memory (see Algorithm~\ref{alg:estimate_density}). 
Hence, the algorithm can rule out from $T$ any hypotheses $h\in T$ with $d(S,h) \not\approx d(S,f)$. 
In cases where there are two large disjoint subsets $T_0, T_1\subseteq T$, $|T_0|,|T_1|\geq\regular|T|$ with 
\begin{eqnarray}\label{eq:local_separability}
\max_{h_0\in T_0}d(S,h_0) + \Omega(\alpha)\leq \min_{h_1\in T_1}d(S,h_1),
\end{eqnarray}
 either $T_0$ or $T_1$ can be ruled out from $T.$ 
 So the algorithm is able to delete $\regular|T|$ hypotheses from $|T|$ while only using a small number of memory bits. 
For classes that satisfy the separability property, as will be formalized in the next section, Equation~${(\ref{eq:local_separability})}$ holds. 
In Section~\ref{sec:applications} we also show examples of classes (e.g., decision lists)  that satisfy the separability property.



The algorithm uses an oracle~\footnote{There is such an oracle for any separable class, see Section~\ref{sec:definitions}. For the entire algorithm to be bounded memory, it is assumed that the oracle is also bounded memory, which is true for all the classes presented in this paper.} that provides the following functionality for any subset of hypotheses $T\subseteq \cH$: 
\begin{itemize}
\item If $T$ is  $(\regular,\close)$-tight, the oracle provides a proof for this by returning $h\in\cH$ with $|T\cap B_h(\close)|\geq \regular|T|.$
\item 
If $T$ is not $(\regular,\close)$-tight, the oracle returns  $S\subseteq\cX, T_1, T_0\subseteq T, d_0, d_1\in\mathbb{R}$ with $|S|\geq \regular|\cX|$, $|T_0|,|T_1| \ge \poly(\regular)\cdot |T|$ and $d_1-d_0 \geq \Omega(\regular)\cdot|S|$ such that $h\in T_0$ implies $e(h,S)\leq d_0$ and $h \in T_1$ implies $e(h,S)\geq d_1$.
\end{itemize}

 In Section~\ref{sec:applications} we show how to \emph{efficiently}, both in time and in memory, implement this oracle for specific classes. 
 The algorithm also uses the following subroutines:
(i) \textbf{Is-close$(h,\close,k)$} --- tests whether $h$ is $\close$-close to the correct hypothesis with error exponentially small in $k$. See Algorithm~\ref{alg:is_close}. 
(ii) \textbf{Estimate$(S,\tau,k)$} --- estimates $d(S,f)$ up to an additive error of $\tau$ with error exponentially small in $k$. 
    See Algorithm~\ref{alg:estimate_density}. 
The correctness of the subroutines is stated and proved in the Appendix. 

{\small
{
\begin{minipage}[t]{7cm}
  \vspace{0pt}  
    {
  \begin{algorithm}[H]  \caption{Is-close$(h, \close, k)$}
    \begin{algorithmic}[1]
    \STATE \textbf{Inputs: } $h\in\cH$, $\close>0$ 
\STATE \textbf{Parameter: } integer $k$
\STATE \textbf{Returns: }\\ True if $h$ is $\close$-close to $f$ \\ False if $h$ is not $3\close$-close to $f$
\STATE $j = 0$
\FOR {$i:=1$ to $k$}
\STATE get labeled example $(x,y)$
\IF {$h(x)\neq y$}
\STATE $j += 1$ 
\ENDIF
\ENDFOR
\RETURN  $j/k \leq 2\close$
\end{algorithmic}
\label{alg:is_close}
  \end{algorithm}
  
  }
\end{minipage}%
  \hspace{1pt} 
\begin{minipage}[t]{7cm}
  \vspace{0pt}

  {

  \begin{algorithm}[H]     
  \begin{algorithmic}[1]\label{alg:estimate_density}
  \caption{Estimate$(S,\tau,k)$}
\STATE \textbf{Inputs: }  $S\subseteq\cX$,$\tau>0$
\STATE \textbf{Parameter: } integer $k$
\STATE \textbf{Output: } $d(S,f)\pm \tau$
\FOR{ $i:=1$ \TO $2k/\tau$  }
\STATE get labeled example $(x,y)$ 
\IF {$x\in S$}
\STATE $counter_S := counter_S +1$ 
\IF { $y=1$}
\STATE $counter_1 := counter_1 +1$ 
\ENDIF
\ENDIF
\ENDFOR
\RETURN $counter_1/counter_S$
\end{algorithmic}
  \end{algorithm}  
 
  }
\end{minipage}
}
}
  \vspace{1pt}

The general bounded-memory algorithm is described in Algorithm~\ref{alg:BMRM} and a graphical representation of it as a branching program appears in Figure~\ref{fig:branching_program_general_algorithm}.
 The algorithm proceeds as follows.
At each step we maintain a set $T\subseteq\cH$ of candidates to be the correct hypothesis. 
In Line~\ref{alg:BMRM:init} we initialize $T$ to be $\cH,$ the entire hypothesis class. 
In Line~\ref{alg:BMRM:oracle_call} the algorithm calls the oracle and if $T$ is $(\regular,\close)$-tight, in Line~\ref{alg:BMRM:is_close} it tests whether the hypothesis $h$ returned by the oracle is $\epsilon$-close to the correct hypothesis.
If $h$ is close enough to the correct hypothesis, the algorithm halts and returns $h,$ otherwise the algorithm can safely remove all the hypotheses that are $\close$-close to $h$, as done in Line~\ref{alg:BMRM:delete_close_to_wrong}. 
If $T$ is not $(\regular,\close)$-tight, then in Line~\ref{alg:BMRM:oracle_call2} the algorithm calls the oracle to find $S, T_1, T_0, d_0, d_1$. 
In Line~\ref{alg:BMRM:estimate} the algorithm estimates $d(S,f)$, where $f$ is the correct hypothesis, and tests whether it is closer to $d_0$ or $d_1.$  Then, it is able to delete $T_0$ or $T_1$ according to the value of $r$ in Lines~\ref{alg:BMRM:delete_T0}, \ref{alg:BMRM:delete_T1}.

{\small 
  \vspace{2pt}
  {

  \begin{algorithm}[H]       \caption{General Bounded Memory Algorithm}
  \begin{algorithmic}[1]\label{alg:BMRM}
\STATE \textbf{Input:} class $\cH$
\STATE \textbf{Parameter: } integer $k$
\STATE $T :=\cH$ \label{alg:BMRM:init}
\LOOP {}
\IF {oracle returns $T$ is $(\regular,\close)$-tight with proof $h$ (i.e., $|T\cap B_h(\close)|\geq \regular|T|$)} \label{alg:BMRM:oracle_call}
\IF {Is-close$(h,\close,k)$} \label{alg:BMRM:is_close}
\RETURN $h$
\ELSE 
\STATE $T := T\setminus B_h(\close)$ \label{alg:BMRM:delete_close_to_wrong}
\ENDIF
\ELSE
\STATE \label{alg:BMRM:oracle_call2} oracle finds $S, T_1, T_0, d_0, d_1$ as stated earlier 
\STATE $r:=\text{Estimate}(S,(d_1-d_0)/2,k)$ \label{alg:BMRM:estimate}
\IF {$r|S| > \frac{d_1+d_0}{2}$}
\STATE \label{alg:BMRM:delete_T0} $T:=T\setminus T_0$
\ELSE
\STATE $T:=T\setminus T_1$ \label{alg:BMRM:delete_T1}
\ENDIF
\ENDIF
\ENDLOOP
\end{algorithmic}
  \end{algorithm}
  }
  }
    \vspace{2pt}

In the next section we introduce the separability property, which uses similar terms as in mixing.  
For classes that satisfy this property the oracle can be implemented and therefore the algorithm functions correctly. 
We prove that several natural classes satisfy this property, for example, the class of decision lists. Thus, it can be learned with bounded memory.  
An interesting feature of the general bounded-memory algorithm is that it is also a statistical query algorithm, as we describe next, and thus robust to random noise.

\begin{figure}
\centering
\begin{subfigure}[b]{.5\textwidth}
  \centering
  \includegraphics[width=.7\linewidth]{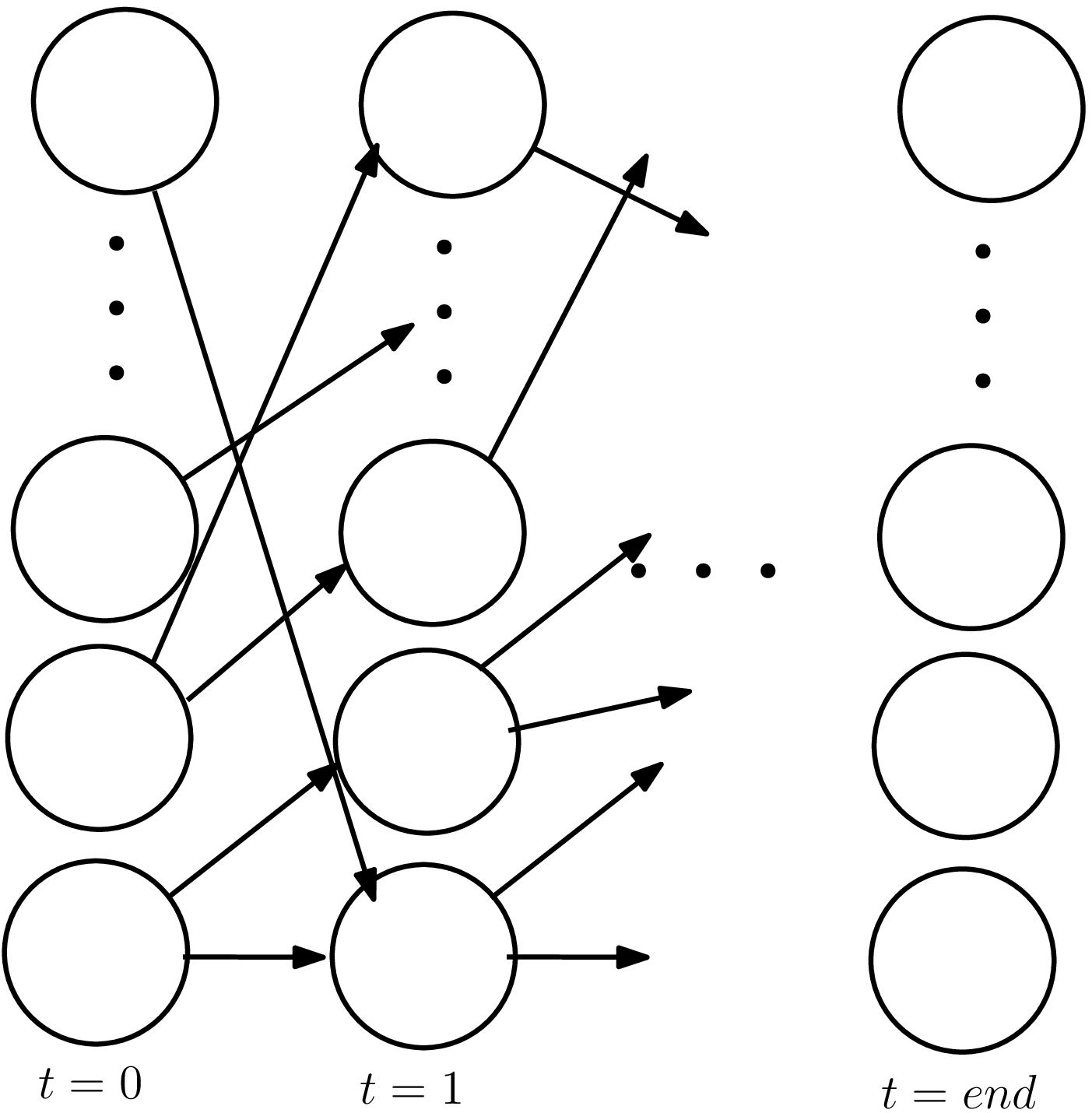}
  \caption{A bounded-memory algorithm}
  \label{fig:branching_program}
\end{subfigure}%
\begin{subfigure}[b]{.5\textwidth}
  \centering
  \includegraphics[width=.7\linewidth]{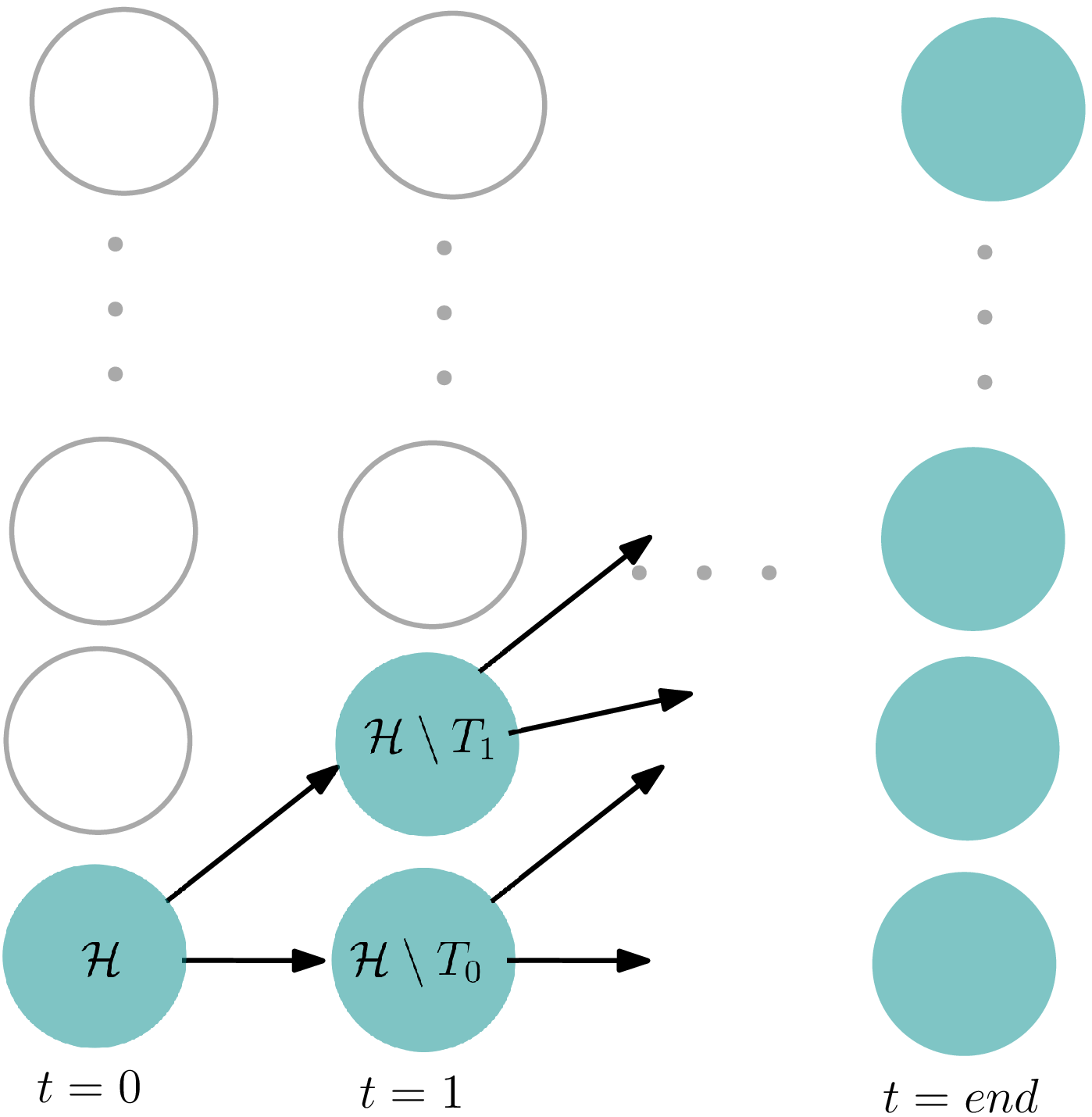}
  \caption{The general bounded-memory algorithm}
  \label{fig:branching_program_general_algorithm}
\end{subfigure}
\caption{(a) Graphical representation of a bounded memory learning algorithm as a branching program (b) Graphical representation of the general bounded memory algorithm. Cyan: memory states that can be reached. From each memory state, there are exactly two other different memory states that can be reached}
\label{fig:test}
\end{figure}

\subsection*{Statistical Queries}
The statistical queries (SQ) framework was introduced by Kearns (see \cite{kearns98}) as a way of designing learning algorithms that are robust to random classification noise. 
A statistical query algorithm access the labeled examples only through 
 \begin{inparaenum}[1)]
\item queries of the form $\psi:\cX\times\{0,1\}\rightarrow\{0,1\}$ and
\item answers $\E_{(x,y)}[\psi(x,y)]$ with an additive error $\tau$.
\end{inparaenum}
Algorithm~\ref{alg:BMRM} access the labeled examples only in the subroutines \emph{Is-close} and \emph{Estimate} in Lines~\ref{alg:BMRM:is_close} and \ref{alg:BMRM:estimate}. 
We prove in the Appendix that 
 indeed these subroutines can be easily implemented in the statistical queries framework. In fact, each subroutine can be implemented using just one statistical query. Thus, the general bounded memory algorithm (Algorithm~\ref{alg:BMRM}) is robust to noise. 

\section{Separable Classes}\label{sec:definitions}
In this section we formally define the separability property as a combinatorial condition of a bipartite graph. 
This property is closely related to the mixing property defined in \cite{moshkovitz17b}.
Recall that hypothesis class can be viewed as a bipartite graph using the hypotheses graph. 
We first define closeness and tight in the language of graph theory. 

In what follow we fix a bipartite graph $(A,B,E)$.
Two vertices $h_1,h_2\in A$ are \emph{$\close$-close} if $|N(h_1) \triangle N(h_2)| \le \close|B|$, where $N(h)$ denotes the set of neighbors of vertex $h$ and $ \triangle$ denotes the symmetric difference.
Similar to the previous section we define the \emph{$\close$-ball} with center $h\in A$ as the set
$B_h(\close) = \{h'\in A \,\vert\, h' \text{ and } h \text{ are } \close\text{-close}\}.$ 
A subset $T\subseteq A$ is $(\regular,\close)$-\emph{tight} if there is a hypothesis $h$ with $|T\cap B_h(\close)|\geq \regular|T|$.

In the previous section we used Equation~$(\ref{eq:local_separability})$, which is a local requirement on each hypothesis. 
We relax this requirement and instead require a weaker global requirement, as stated in the following definition.
\begin{definition}[$(\regular,\close)$-separability]\label{dfn:strongly-non-mixing}
We say that a bipartite graph $(A,B,E)$ is \emph{$(\regular,\close)$-separable} if for any $T\subseteq A$ that is not $(\regular,\close)$-tight there are subsets $S\subseteq B$ and $T_0,T_1\subseteq T$ with $T_0\cap T_1=\emptyset$, $|S|\geq \regular|B|, |T_0|\geq \regular|T|, |T_1|\geq \regular|T|$  such that $\left|d(S,T_0)-d(S,T_1)\right|\geq\regular.$
\end{definition}
A hypothesis class is \emph{$(\regular,\epsilon)$-separable} if its hypotheses graph is $(\regular,\close)$-separable.
Perhaps surprisingly, the global requirement of separability implies the local requirement, as stated in the following claim (see the Appendix for the proof). 
%


\begin{claim}\label{clm-differentiator-then-vertex}
Let $(A,B,E)$ be a bipartite graph. For any $T\subseteq A$ that is $\regular$-separable there are $S\subseteq B$ with $|S|\geq \regular|B|$, $T_0,T_1 \subseteq T$ with $|T_0|,|T_1| \ge \frac12\regular^2|T|$ and $d_0,d_1 \in \mathbb{R}$ with $d_1-d_0 \geq\frac{\regular}{4}|S|$ such that $h\in T_0$ implies $e(h,S)\leq d_0$ and $h \in T_1$ implies $e(h,S)\geq d_1$.
\end{claim}


The next theorem proves the correctness of Algorithm~\ref{alg:BMRM} (the proof appears in the Appendix). 
For brevity we define \emph{$(m,b,\delta,\close)$-bounded memory learning algorithm} as an algorithm that uses at most $m$ labeled examples sampled from the uniform distribution, $b$ bits of memory, and returns a hypothesis that is $\epsilon$-close to the correct hypothesis with probability at least $1-\delta.$ We omit the $O$ symbol for simplicity. 


\begin{theorem}\label{thm-learning-bounded-memory}
For any hypothesis class $\cH$ that is $(\regular,\close)$-separable, Algorithm~\ref{alg:BMRM} is a $$\left(\log|\cH|\cdot\frac{\log\log|\cH|+\log\nicefrac{1}{\regular}}{\regular^5},\, \log|\cH|\cdot\frac{1}{\regular^2},\, 0.1,\,\close\right)-$$$\text{bounded memory algorithm for }\cH.$
%
%
\end{theorem}

If $\regular^{-2}$ is smaller than $\log|\cX|$ we get that Algorithm~\ref{alg:BMRM} is indeed a bounded-memory algorithm.
On the other hand, the number of samples is increased by a small factor of $\nicefrac{(\log\log|\cH|+\log\nicefrac{1}{\regular})}{\regular^5}$ compared to the $\log|\cH|$ examples needed in case $\VC{\cH}\approx\log|\cH|$. 
One might wonder how the number of samples does not depend on $\close$. Taking a closer look, we observe that this is not the case  as generally $\epsilon$ is lower bounded by a function of $\alpha$. Take for example $T$ that contains hypotheses that disagree on exactly $\alpha^2$ of the examples. 
If $\epsilon<\alpha^2$, then $T$ is not $(\alpha, \epsilon)$-tight thus the class is not $(\alpha, \epsilon)$-separable. 


\section{Separable Classes: Examples}\label{sec:applications}
In this section we present a few natural classes and prove they are separable. 
This implies, using Theorem~\ref{thm-learning-bounded-memory}, that they are properly learnable with bounded memory. 

\subsection{Threshold Functions}\label{subsec:threshold_functions}
 The class of threshold functions in $[0,1]$ is $\{h_b:[0,1]\rightarrow\{0,1\} : b\in[0,1]\}$ and $h_b(x)=1 \Leftrightarrow x\leq b.$ 
 The class of discrete thresholds $\cH_{TH;n}$ is defined similarly but over the discrete domain $\cX$ of size $n$ with $\cX=\left\{\frac{1}{n},\frac{2}{n},\ldots,\frac{n}{n}=1\right\}$  and $b\in\left\{\frac{1}{2n},\frac{3}{2n},\ldots,\frac{2n+1}{2n}\right\}.$ This class is known to be easily learnable in the realizable case  by simply taking the largest example with label $1$. 
This simple algorithm uses $O(\log n)$ bits and $O(1/\epsilon)$ examples for accuracy $\close$ and constant confidence. 
We use this class  to demonstrate how to
 \begin{inparaenum}[a)]
\item prove that a class is separable 
\item implement the oracles required by the general bounded-memory algorithm, see Section~\ref{sec:algorithms}.
\item use a few simple tricks that enables designing a faster implementation of the general algorithm. 
\end{inparaenum}

One can prove that the class $\cH_{TH; n}$ is $(\regular,\regular)$-separable,  for any $0<\regular<1/3$.
The proof appears in the Appendix and 
the main ideas are presented in Figure~\ref{fig:thershold_functions}. 
\begin{figure}
\begin{center}
 \includegraphics[scale=0.9]{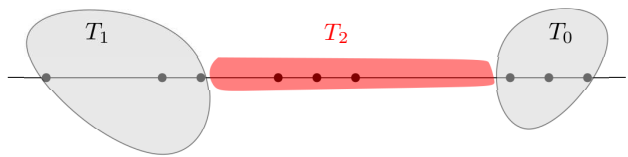}
 \end{center}
\caption{The solid line represents the interval $[0,1]$ and the dots represent the hypotheses in $T$. The subset $T_0\subseteq T$ ($T_1\subseteq T$) represents the $\regular|T|$ largest (smallest) hypotheses in $T$. All the other hypotheses in $T$ are denoted by $T_2.$ If $T_2$ is contained in an interval of length smaller than $\regular$, then $T$ is $(\regular, \regular)$-tight. Otherwise, $d(S,T_1)-d(S,T_0)\geq\regular$ for $S=\cX.$}
\label{fig:thershold_functions}
 \end{figure}
There are a few tricks we can use to design a more efficient implementation of the general algorithm. 
During the run of the algorithm the only possible $T$'s are intervals. 
This means that only two values are needed in order to describe $T.$
If the length of  $T=[a_1,a_2]$ is at most $\close$ (i.e. $|a_2-a_1|\leq\close$) then we are done, as $T\subseteq B_{h_{(a_1+a_2)/2}}(\close).$
At each step of the algorithm, we define $S$ to be a large interval in the middle of $T,$ namely $S=\left[a_1+\frac{a_2-a_1}{3},a_2-\frac{a_2-a_1}{3}\right].$
Note that $|S|\geq\frac{\close}{3}|\cX|.$ We define
 $T_0 = \{ h_b\in T : b < a_1+\frac{a_2-a_1}{3} \}$ and $T_1 = \{ h_b\in T : b > a_2-\frac{a_2-a_1}{3} \}.$
Note that since $d(S,T_0)=0$ and $d(S,T_1)=1$ only one sample from $S$ suffices to decide whether to delete $T_0$ or $T_1$ in the PAC framework (the realizable case). 
The algorithm uses $O(\log n)$ bits (to save $a_1$ and $a_2$) and $O(\frac1\close\log\frac1\close)$ examples.
The algorithm runs in time $O(\log n + \frac1\close\log \frac1\close)$.

\subsection{Equal-Piece Classifiers}\label{subsec:equal_piece} 
Each hypothesis in the class equal-piece classifiers $h\in \cH_{EP; p}$ corresponds to a (disjoint) union of intervals each of length exactly $p<1$, that is, $\bigcup[a^h_i, a^h_i+p]$ and $h(x)=1$ if and only if $x$ is inside one of these intervals.
More formally, the examples are the numbers $\cX=\left\{\frac{1}{n},\frac{2}{n},\ldots,\frac{n}{n}=1\right\}$ and the hypotheses $h\in\cH_{EP; p}$ correspond to the parameters $a^h_1,a^h_2,\ldots,a^h_k$ with $a^h_1+p< a^h_2,\ldots,a^h_{k-1}+p< a^h_k, a^h_k+p<1$ and they define the intervals 
$[a^h_1,a^h_1+p],[a^h_2, a^h_2+p],\ldots,[a^h_k, a^h_k+p].$
An example $x\in\cX$ has $h(x)=1$ if and only if there is $1\leq i\leq k$ such that $x\in [a^h_{i}, a^h_{i} + p].$

Note that the class $\cH_{EP; p}$ is quite complex since it is easy to verify that it has a VC-dimension of at least $1/p$ (partition $[0,1]$ into $p$ consecutive equal parts and take one point from each part --- this set is shattered by $\cH_{EP; p}$). The next theorem shows the class is separable and thus, by Theorem~\ref{thm-learning-bounded-memory}, Algorithm~\ref{alg:BMRM} can be applied to equal-piece classifiers. 
Time-efficient implementation can be found in the Appendix.

\begin{theorem}\label{clm:application:equal_piece}
For any $\regular,\close\in(0,1)$ with $\close<\nicefrac{24}{p}$ and $\frac{2}{|\cX|}<\regular<\frac{p^2\close}{24}$ the class $\cH_{EP; p}$ is $(\regular,\close)$-separable. 
\end{theorem}


%

\subsection{Decision Lists}\label{subsec:applications:decision_lists}
A decision list is a function defined over $n$ Boolean inputs of the following form:
\[\textbf{if } \ell_1 \textbf{ then } b_1 \textbf{ else if } \ell_2 \textbf{ then } \ldots \textbf{if }\ell_k \textbf{ then } b_k \textbf{ else } b_{k+1},\]
where $\ell_1,\ldots,\ell_k$ are literals over the $n$ Boolean variables and $b_1,\ldots,b_{k+1}$ are bits in $\{0,1\}.$
We say that the $i$-th \emph{level} in the last expression is the part ``$\textbf{if }\ell_i \textbf{ then } b_i$'' and the literal $\ell_i$  \emph{leads} to the bit $b_i$. Given some assignment to the $n$ Boolean variables we say that the literal $\ell_i$ is \emph{true} if it is true under this assignment. 
Note that there is no need to use the same variable twice in a decision list. In particular, we can assume without loss of generality that $k = n$. 
Denote the set of all decision lists over $n$ Boolean inputs by $\cH_{DL;n}.$  
Note that \[|\cH_{DL;n}|\leq n!\cdot 4^n \Rightarrow \log|\cH_{DL;n}|\leq n\log n + 2n\]
 (the first inequality is true since at each level we need to choose a variable that was not used before, to decide whether it will appear with a negation or not and if true, whether it will lead to $0$ or $1$).

\begin{theorem}\label{thm:applications:decision_list}
For any $\close\in(2^{-n},1)$ the class $\cH_{DL;n}$  is $\left(\min\left\{\frac1{200n^4},\close\right\},\close\right)$-separable. 
\end{theorem}

In the appendix we show an efficient implementation both in time and in space of the main algorithm to the case of decision lists. 
Rivest \cite{rivest87} described a learning algorithm for this class. 
However, the suggested algorithm saves all the given examples, and thus does not qualify as a bounded-memory algorithm.
Rivest's algorithm uses $O(n^2 \log n\cdot \nicefrac1\epsilon)$ memory bits and $O(n\log n\cdot\nicefrac1\epsilon)$ examples
to learn with constant confidence parameter and accuracy parameter $\epsilon$. 
 Our algorithm is able to learn this class with only $O(n\log\frac{1}{\epsilon})$ bits of memory, which is a quadratic improvement in $n$ as well as an improvement in $1/\epsilon$.
 Our algorithm introduces an increase in the number of examples to $O\left({n\log n}\cdot\nicefrac{1}{\epsilon^2} \log\nicefrac1\close\right)$, but the increase only depends on $1/\epsilon$. 
In the Appendix you can find an efficient implementation, both in space and in time, of the general algorithm for the the class $\cH_{DL;n}$. 

Several works described algorithms for learning decision lists \cite{nevo02, dhagat94, klivans06}, however those works learn this class under several assumptions. Ignoring those assumptions leaves their algorithm with a non-polynomial number of examples, which is  a minimal requirement  for any learning algorithm.
In a different work \cite{long07} 
 limit themselves to the uniformly distributed examples scenario, as our general algorithm does.
They design a weak learner and use the MadaBoost algorithm \cite{domingo00} which is a boosting-by-filtering algorithm \cite{schapire90}. 
As a consequence of using the MadaBoost algorithm they use assumptions stated in  \cite{domingo00}. Further, they improper learn the class while our general algorithm provides a proper bounded-memory learning algorithm. 

\section{Discussion and Open Questions}
In this paper we introduced the concept of separability, which is similar to non-mixing, suggested a general bounded memory algorithm, proved its correctness for classes that are separable, and  proved it is also an sq-algorithm. 
We also derived time-efficient bounded memory algorithms for three natural classes: threshold functions, equal-piece classifiers, and decision lists.

Several questions remain open. One question is to prove that other classes are separable and thus can be learned with the general bounded memory algorithm. 
A second open problem is to extend our work to the case of an unknown distribution over the examples. Another open problem is to close the gap between separability and non-mixing. 



\bibliography{learnability}
\bibliographystyle{plain}

\appendix
\section{Outline}
This appendix contains several technical claims: 
 the classes presented in the main text are separable (Section~\ref{apx:class_are_separable}), 
 time-efficient implementation of the general algorithm to these classes (Section~\ref{apx:time_efficent_implementations}), 
 correctness proof of the general algorithm for classes that satisfy separability (Section~\ref{apx:technical_proofs: Proofs of Claims from Section 3 in the Main Text}), and a proof that the general algorithm is also a statistical query algorithm ({Section~\ref{apx:general_alg_sq_algorithm}}).

\section{Separability}\label{apx:class_are_separable}
In this section we formally prove that all classes introduced in the main text are in fact \emph{separable}. 

\subsection{The class of threshold functions is separable}
\begin{theorem}
For any $0<\regular<1/3$, the class $\cH_{TH; n}$ is $(\regular,\regular)$-separable.
\end{theorem}
\begin{proof}
Fix $\regular\in(0,1)$ and $T\subseteq\cH_{TH;n}$ which is not $(\regular,\regular)$-tight.  
We want to find $S\subseteq\cX$ with $|S|\geq\regular|\cX|$ and $T_0,T_1\subseteq T$ with $|T_0|,|T_1|\geq \regular|T|$ such that 
\begin{eqnarray}\label{eq:applications_discrete_threshold_functions_want_to_prove}
|d(T_1,S)-d(T_0,S)|\geq\regular
\end{eqnarray}
Take $S=\cX,$  which immediately implies that $|S|\geq\regular|\cX|.$
Denote by $t_0\in[0,1]$ the minimal value such that for $\regular|T|$ of  the hypotheses $h_b\in T$ it holds that $b\leq t_0$  and by $t_1\in[0,1]$ the maximal value such that for $\regular|T|$ of the hypotheses $h_b\in T$ it holds that $b\geq t_1.$ 
Take $T_0=\{h_b\in T : b\leq t_0\}$ and $T_1=\{h_b\in T : b\geq t_1\}$,  which immediately implies that $|T_0|,|T_1|\geq\regular|T|.$
Notice that for any $h_b$ in the class it holds that $d(h_b,S)=b.$
To prove~(\ref{eq:applications_discrete_threshold_functions_want_to_prove}), note that 
\begin{eqnarray*}
|d(T_1,S)-d(T_0,S)| &\geq&   t_1 - t_0. 
\end{eqnarray*}
Assume by contradiction that $t_1-t_0<\regular.$ Then, since for each $h_b\in T\setminus(T_0\cup T_1)$ it holds that $b\in(t_0,t_1)$ we get that 
$$T\setminus(T_0\cup T_1)\subseteq B_{\frac{(t_1+t_0)}{2}}\left(\frac{\regular}{2}\right)$$

Since $\regular\leq1/3$ we get that $1-2\regular\geq \regular$, thus $|T\setminus(T_0\cup T_1)|\geq\regular|T|$ which is a contradiction to the assumption that $T $ is not $(\regular,\regular)$-tight.

\end{proof}


\subsection{The class \texorpdfstring{$\cH_{DL;n}$}{TEXT} is separable}
\begin{theorem}
For any $\close\in(2^{-n},1)$ the class $\cH_{DL;n}$  is $\left(\min\left\{\frac1{200n^4},\close\right\},\close\right)$-separable. 
\end{theorem}
\begin{proof}
Fix $\epsilon\in(0,1)$ and $T\subseteq\cH_{DL; n}$ that is not $(\min\{\frac1{200n^4},\close\},\close)$-tight. 
To prove that $\cH_{DL; n}$ is $\left(\min\{\frac1{200n^4},\close\},\close\right)$-separable we will show that $T$ is not $\min\{\frac1{200n^4},\close\}$-tight. 
To show that we will find two literals $\ell^0,\ell^1$ and $T_0,T_1\subseteq T$ with $|T_0|,|T_1|\geq\frac{1}{200n^4}|T|$ that have the following properties $(\star)$:
\begin{enumerate}
    \item For all hypotheses in $T_0$: \label{item:app:DL_item1}
\begin{itemize}
    \item $\ell^0$ leads to $0$
    \item $\ell^0$ appears at level $i_0\leq\log\frac1\close + 1$
    \item $\ell^0$ is in a lower level than $\ell^1$ \label{item:app:DL_item13}
    

\end{itemize}
\item Similarly for all hypotheses in $T_1$: \label{item:app:DL_item2}
\begin{itemize}
    \item $\ell^1$ leads to $1$
    \item $\ell^1$ appear at level $i_1\leq\log\frac1\close + 1 $
    \item $\ell^1$ is in a lower level than $\ell^0$ \label{item:app:DL_item23}
    \end{itemize}
 \item There is a level $j\leq \min\{i_0,i_1\}$ and a bit $b\in\{0,1\}$ such that all hypotheses in $T_0\cup T_1$ \label{item:app:DL_item3}
 \begin{enumerate}
     \item are identical up to level $j$ \label{item:app:DL_identical}
     \item leads to the same value $b$ in levels $j+1$ to  $\max\{i_0,i_1\}-1$ \label{item:app:DL_same_bit}
 \end{enumerate}
 \end{enumerate}
 Note that for any decision list permuting consecutive literals that all lead to the same bit creates an equivalent decision list; thus, when we write ``identical decision lists'', we mean identical up to this kind of permutation. 
 
The correctness of the last three properties $(\star)$ will finish the proof since we can take $S$ to consist of all the assignments where the literals $\ell^0$ and $\ell^1$ are true. In this case it holds that $|S|\geq |\cX|/4$, the disjoint subsets $T_0,T_1$ are large (i.e., $|T_0|,|T_1|\geq\frac{1}{200n^4}$). To bound $|d(T_1,S)-d(T_0,S)|$ from below, we partition $S$ into two parts: $S_1$ all assignments such that at least one of the literals $\ell_1,\ldots,\ell_j$ are true (recall that level $j$ is defined in Item~\ref{item:app:DL_item3}), and $S_2=S\setminus S_1$. Assume without loss of generality that  bit $b$ defined in Item~\ref{item:app:DL_same_bit} is equal to $0$. Note that  
\begin{eqnarray*}
|d(T_1,S)-d(T_0,S)| &=& \left |\sum_{a\in S} \frac{e(T_1,a)}{|T_1||S|} - \frac{e(T_0,a)}{|T_0||S|} \right| \\
&=& \left |\sum_{a\in S_1} \frac{e(T_1,a)}{|T_1||S|} - \frac{e(T_0,a)}{|T_0||S|}
+ \sum_{a\in S_2} \frac{e(T_1,a)}{|T_1||S|} - \frac{e(T_0,a)}{|T_0||S|} \right|\\
&=&\left|\sum_{a\in S_2} \frac{e(T_1,a)}{|T_1||S|} - \frac{e(T_0,a)}{|T_0||S|} \right|\\
&=&\sum_{a\in S_2} \frac{e(T_1,a)}{|T_1||S|}\\
&\geq& 2^{-\max\{i_0,i_1\}+1}\geq\close,
\end{eqnarray*}
where the third equality follows from Item~\ref{item:app:DL_identical}, the fourth equality follows from Item~\ref{item:app:DL_same_bit}, and the first inequality follows from Item~\ref{item:app:DL_item2} since for each assignment that is false in all literals that appear before level $i_1$ we have that $e(T_1,a)=|T_1|$ and these assignments constitute a fraction $2^{-(i_1-1)}$ out of the assignments in $S$. The last inequality follows from Items~\ref{item:app:DL_item1},\ref{item:app:DL_item2}.

To prove that there are literals $\ell^0,\ell^1$ and subsets $T_0,T_1$ as desired in $(\star)$, we will prove by induction on level $i$ that if there are not literals $\ell^0,\ell^1$ up to level $i$, then there is a subset $T^{i}\subseteq T$ with $|T^{i}|\geq (1-\frac{1}{4n^2})^i|T|$, a bit $b\in\{0,1\}$ and $j\leq i$ such that for all hypotheses in $T^{i}$ $\quad(\star\star)$
 \begin{itemize}
     \item are identical up to level  $j$
     \item all literals in levels $j+1$ to $i$ lead to the same value $b$ 
 \end{itemize}

Proving $(\star\star)$ will finish the proof because if we do not find $\ell^0,\ell^1$ up to level $i\leq\log\frac1{\close}\leq n$ then we have that
\begin{eqnarray*}
|T^i| &\geq & \left(1-\frac{1}{4n^2}\right)^i|T|\\
& \geq &\left(1-\frac{1}{4n^2}\right)^n|T|\\
& \geq & \left(1-\frac1{4n}\right)|T|, \quad(\star\star\star)
\end{eqnarray*}
where in the second inequality we used the fact that $i\leq n$ and in the third inequality we used the fact that for any natural number $n$, the inequality $(1-x)^n\geq1-nx$ is true.
Thus, we have at least $\left(1-\frac1{4n}\right)|T|$ hypotheses in $T^i\subseteq T$ that are $\close$-close as we explain next, which is a contradiction to the assumption that $T$ is not tight.

Take the decision list $h$ which is exactly the same as all hypotheses in $T^i$ up to level $j$ and it returns $b$ afterwards. We will prove that all the hypotheses in $T^i$ are $\close$-close to $h$. 
Take any hypothesis $h'\in T^i$. 
All examples that cause one of the literals up until level $i$ to be true have the same label as $h$. 
Thus, these hypotheses agree on $\frac{1}{2} + \frac{1}{4} + \ldots + \frac{1}{2^i}=1-\frac{1}{2^i}\geq 1-\epsilon$ fraction of the examples.

We will prove that $(\star\star)$ is true by induction at level $i$. The basis $i=0$ is vacuously true.
If we do not find $\ell^0,\ell^1$ at level $i+1$ and the induction hypothesis holds for any $j\leq i$ then first recall that $i\leq\log\frac1\close$.

To continue, there are a few cases depending on the number of literals in level $i+1$ that lead to $\bar{b},$ where $\bar{b}$ denotes the opposite value of the bit $b$ (i.e., $\bar{0}=1,\bar{1}=0$).

\textbf{Case 1:} If for at least $(1-\frac{1}{8n^2})|T^i|$ of the hypotheses in $T^{i}$ the literal in level $i+1$ leads to $b$ (i.e., the same bit $b$ as in level $i$) then define $T^{i+1}$ as $T^i$ minus all the $(1-\frac1{8n^2})|T^i|$ hypotheses that the literal in level $i+1$ does not lead to the bit $b.$ The induction claim will follow since $1-\frac{1}{8n^2}\geq1-\frac1{4n^2}.$

\textbf{Case 2:} If there are at least $\frac1{8n^2}|T^i|$ hypotheses in $T^i$ such that the literal in level $i+1$ leads to $\bar{b}$. 
Since there are $2n$ literals, there is a literal $\ell$ that is at level $i+1$ in at least $\frac1{16n^3}|T^i|$ of the hypotheses in $T^i$ and lead to $\bar{b}.$ Denote this set of hypotheses by $W.$ 

Call a literal \emph{useful} if in at least $\frac{1}{16n^3}|T^i|$ of the hypotheses in $T^{i}$ it appears in levels $j+1$ to $i+1$ and it leads to $b.$ Note that there must be at least $i-j$ useful literals (using  Claim~\ref{clm:app:DL:counting} with $p=|T^i|, m=i-j$, and the fact that $\frac{|T^i|}{2n}\geq\frac{|T^i|}{16n^3}$). Again there are a few cases depending on whether there are $i-j$ useful literals or more. 

\textbf{Case 2.1:}
If there are exactly $i-j$ useful literals we will define $T^{i+1}$ by removing all hypotheses in $T^i$ that are one of the following types
\begin{itemize}
    \item contains a literal that is not useful at some level from $j$ up to $i$.
    \item contains a literal that leads to $b$ at level $i+1$.
\end{itemize}
To prove the induction claim we need to prove that $T^{i+1}$ is large. 
If there are $\frac{1}{16n^3}|T^i|$ hypotheses that the literal in level $i+1$ leads to $b$, then we would have that there are more than $i-j$ useful literals (see Claim~\ref{clm:app:DL:counting}).
Thus, $$|T^{i+1}|\geq \left(1-\frac{1}{16n^3}\cdot2n-\frac{1}{16n^3}\cdot 2n\right)|T^i|\geq\left(1-\frac{1}{4n^2}\right)|T^i|.$$ 
Note that the hypotheses in $T^{i+1}$ all contain the same $i-j$ useful literals in levels $j$ to $i$. Hence all of the hypotheses are identical (up to permutation) from levels $j$ to $i$. From the induction hypothesis we get that all the hypotheses are identical up to level $i$. 
Note also that for all hypotheses in $T^{i+1}$, the literal in level $i+1$ leads to the same value $\bar{b}.$
Hence, we proved the induction hypothesis. 

\textbf{Case 2.2:}
If there are more than $i-j$ useful literals then there must be a useful literal $\ell'$ that appear in at most $(1-\frac{1}{3n})|W|$ of the hypotheses in $W$ at a level smaller than $i$ (see Claim~\ref{clm:app:DL:counting2} with $W$ of size $p=|W|$ and $m=i-j$).
In this case we will show how to choose $\ell^0,\ell^1,T_0,T_1$ that will fulfill $(\star)$.  
Pick $\ell^{\bar{b}}=\ell$, $\ell^b=\ell'$ and $T_{b}$ the $\frac{1}{16n^3}|T^i|$ hypotheses that make the literal $\ell$ useful and $T_{\bar{b}}$ the hypotheses where $\ell'$ does not appear before $\ell$. Notice that  
$$|T_{\bar{b}}|\geq\frac1{3n}|W|=\frac1{48n^4}|T^i|\geq\frac{1}{48n^4}\left(1-\frac{1}{4n}\right)|T|\geq\frac{1}{200n^4}|T|,$$
where the second inequality follows from $(\star\star\star).$

If $\ell,\ell'$ do not share the same variable (i.e. $\bar{\ell}\neq\ell'$) then properties $(\star)$ are fulfilled. 
Otherwise, properties \ref{item:app:DL_item13}, \ref{item:app:DL_item23} in $(\star)$ do not hold as $\ell=x_r,\ell'=\bar{x_r}$, for some variable $x_r$, do not appear in the same decision list and hence do not appear one before another.  If there are more than $i-j$ useful literals in levels $j+1$ to $i$ then we can continue as before with $\ell'$ not equal to $x_r.$
 If there are exactly $i-j$ useful literals in levels $j+1$ to $i$, then, as in Case 2.1, we can prove that that there is a large subset of $T'\subseteq T^i$ that are similar up to level $i.$
 Divide $T'$ into two subset $T'_0, T'_1$ depending on the literal on level $i+1$ ($x_r$ or $\bar{x_r}$).
 Take $S$ such that $\ell'$ is $True$ and all the literals in $T'$ are $False.$ In this case $d(S,T'_0)=0$ (without loss of generality we can assume that $b=0$).
 If $|d(S,T'_0)-d(S,T'_1)|$ is not large enough (i.e., at least $\close$), then $d(S,T'_1)\leq \close$.
By Markov's inequality, for at least half of the hypotheses $h\in T'_1$ it holds that  $d(S,h)\leq 2\close$.
Take $h'$ to be equal to the hypotheses in $T'_1$ up until level $i+1$ and then returns $0$.
Note that that, by the choice of $S$, that at least half of $T'_1$ are in $B_{h'}(\close)$ (each assignment not in $S$ evaluates to the same value in $T'_1$ and $S$ contains at most half of the all possible assignments).
This in turn implies that $T$ is $(\min\{\frac1{200n^4},\close\},\close)$-tight, which is a contradiction.

Note that to implement the oracle needed for the general bounded-memory algorithm requires $O(n\cdot\log\frac1\close)$ bits. 
\end{proof}

\begin{claim}\label{clm:app:DL:counting}
For any matrix $A$ of size $m\times p$ where each cell in $A$ is an integer in $[n]$ and each integer in $[n]$ does not appear in $A$ more than $p$, then there must be at least $m$ integers in $[n]$ that appear at least $\frac{p}{2n}$ times in $A$.  
\end{claim}
\begin{proof}
Assume by contradiction that there are at most $m-1$ integers in $[n]$ that appear at least $\frac{p}{2n}$ times in $A$. By the assumption in the claim, each of these integers can appear at most $p$ times in $A$. All of the other $n-(m-1)$ integers appear at most $\frac{p}{2n}$ times in $A$. Thus, we have that all the numbers occupy at most $$(m-1)p + (n-(m-1))\frac{p}{2n}=\left(m-1 + \frac{n-m+1}{2n}\right)p<mp,$$
which is a contradiction to the assumption that only integers in $[n]$ appear in the $mp$ cells of $A.$  
\end{proof}

\begin{claim}\label{clm:app:DL:counting2}
For any matrix $A$ of size $m\times p$ where each cell in $A$ is in an integer in $[n]$, $m\leq n$ then the number of integers in $[n]$ that appear in at least $(1-\frac1{3n})p$ times in $A$ is at most $m$.   
\end{claim}
\begin{proof}
Assume by contradiction that there are at least $m+1$ integers in $[n]$ where each appear at least $(1-\frac1{3n})p$ times in $A.$ Then these integers cover at least $(m+1)(1-\frac1{3n})p>mp$ cells in $A$ which is a contradiction to the size of $A$. 
\end{proof}


\subsection{Equal-piece classifiers is separable }
\begin{theorem}
For any $\regular,\close\in(0,1)$ with $\close<\nicefrac{24}{p}$ and $\frac{2}{|\cX|}<\regular<\frac{p^2\close}{24}$ the class $\cH_{EP; p}$ is $(\regular,\close)$-separable. 
\end{theorem}
\begin{proof}
Fix $\frac{2}{|\cX|}<\regular<\frac{p^2\close}{24}$ and $T\subseteq\cH_{EP; p}$ that is not $(\regular, \close)$-tight. 

We will show that there is a set $S\subseteq\cX$ 
with $|S|\geq\regular|\cX|$ and $T_1, T_2\subseteq\cH$ with $|T_1|,|T_2|\geq\regular|T|$ such that $d(S,T_1)=1$ and $d(S,T_2)=0$; thus, in particular $|d(S,T_1)-d(S,T_2)|=1>\regular$ which will prove the claim.
We will in fact prove that there is an open interval $I\subseteq[0,1]$ of length $\alpha':=2\regular$ such that the sets $T_1 = \{h\in T | \exists k, I \subseteq[a^h_k, a^h_k + p]\}$ and $T_2=\{h\in T| \forall k, I \cap[a^h_k, a^h_k + p]=\emptyset\}$ satisfying $|T_1|,|T_2|\geq2\regular|T|.$ 
We call such a $I$ \emph{separating}. 
 This will prove that $\cH_{EP; p}$ is $(\regular,\close)$-separable since for $S = I\cap\cX$ we have that $|S|\geq\regular|\cX|$ (from the assumption in the claim regarding the upper bound on $\cX$), $d(S,T_1)=1$, and $d(S,T_2)=0.$
For ease of notation we replace $\regular'=2\regular$ by $\regular$ from now on.  
 
Next, we will prove something even stronger by showing that there is a sequence $0 = u_0\leq u_1 \leq\cdots$ and a sequence of sets $T = T^0\supseteq T^{1}\supseteq\cdots$ such that for all $i\geq 0$ if there is no separating $I$ in the ``window'' $$W_i := [u_i,u_i+p-i\regular ]$$ and there is no separating $I$ in the previous windows either, then the following four properties are satisfied:
 
 \begin{enumerate}
\item $|T^{i}|\geq (1-2i\regular)|T|$\label{item:piecewise_large_T}
\item $u_{i}\geq ip-\regular\sum_{j=1}^{i}j$ \label{item:piecewise_disjoint_window}
\item every $h_1,h_2\in T^{i}$ is similar up to the current window: $$|\{x \in \cX\cap [0, u_i] | h_1(x)\neq h_2(x)\}|\leq \regular\sum_{j=1}^{i+2}j|\cX|$$\label{item:piecewise_large_agre}
\item no hypothesis $h\in T^{i}$ has an endpoint in the current window: for all $h\in T^{i}$ and $k$ it holds that \label{item:piecewise_no_end_point}
$$a^h_k+p\notin [u_i,u_i+p-i\regular].$$
\end{enumerate}  
Assume by contradiction that there is no separating $I$. 
Hence, we deduce that the four properties hold for all windows $W_1,W_2,\ldots$
By Property~\ref{item:piecewise_disjoint_window}, for $\ell \leq \lceil \frac{2}{p} \rceil<\frac3p$,
it holds that $u_\ell\geq 1.$
Fix $h\in T^\ell.$
By Property~\ref{item:piecewise_large_T} we know that
\begin{eqnarray*}
 |T^{\ell}| &\geq& |T|\left(1-\frac{6\regular}{p}\right)\\
 (\text{since } \regular<\nicefrac{p}{18})\;\;\;\; & >&  |T|\cdot\frac23 \\
 & >&  \regular |T| \\
\end{eqnarray*}
Take any $h'\in T^{\ell}.$ Since $\cX\subseteq[0,1]$, Property~\ref{item:piecewise_large_agre} implies that  $h'\in B_h(\close)$ (since $\frac{12\regular}{p^2}<\close$), which is a contradiction to $T$ not being tight. 
 
To complete the proof we prove by induction on $i$ that if there are no separating $I$ in the windows up to $W_i$ then there are $T^i, u_i$ that have the previous four properties.

\underline{Induction basis:}
Since $T^0=T$ and $u_0=0$, Properties~\ref{item:piecewise_large_T} and \ref{item:piecewise_disjoint_window} hold. 
Since $0\not\in\cX$, Property~\ref{item:piecewise_large_agre} holds.
There is no endpoint in the interval $W_0=[0,p]$ (since the length of each interval in any hypothesis in $\cH_{EP; p}$ is $p$) and since $0\not\in\cX$ we can assume that for each $h$ and $k$ it holds that $a^h_k>0$ thus proving Property~\ref{item:piecewise_no_end_point} holds. 

\underline{Induction step:}
By the induction hypothesis there is no separating $I$ in the full range of $W_i$. To use it we intuitively move a small sliding-window $I$ of length $\regular$ in the current window $W_i$.
For each such $I$ we can calculate the number of hypotheses in $T^i$ that contain $I$,
  $$ c_1^I = |\{h\in T^i| \exists k, I\subseteq [a^h_k, a^h_k + p]\}|$$
  and the number of hypotheses in $T^i$ that do not intersect $I$,
  $$ c_0^I = |\{h\in T^i| \forall k,  I \cap [a^h_k, a^h_k + p]=\emptyset\}|.$$
  Note that $I$ is separating if and only if $c^I_0 \geq \regular|T|$ and $c^I_1\geq \regular |T|.$ 
  Thus, our assumption is that for every $I\subseteq W_i$ either $c^I_0 < \regular|T|$ or $c^I_1 < \regular |T|.$ 
 Observe that by Property~\ref{item:piecewise_no_end_point} there are no endpoints of $T^i$ in $W_i$, which immediately implies that $c_1^I$ can only increase 
 as we slide $I$ within $W_i$. 
 We next consider two cases depending on whether there is $I\subseteq W_i$ with $c_1^I\geq\regular|T|$ or not. In each case we need to define $u_{i+1}, T^{i+1}$ and prove that the four properties hold for $i+1$ which will complete the proof.

  \begin{itemize} 
  \item \textbf{Case 1:} there is no $I\subseteq W_i$ with $c_1^I\geq\regular|T|$, i.e., $c_1^I$ is always smaller than $\regular |T|$ as we slide $I.$
Define $$T^{i+1}=\{h\in T^i |\forall I\subseteq W_i \;\forall k, I\nsubseteq[a_k^h,a_k^h+p]\}\quad\text{ and }\quad u_{i+1}=u_i+|W_i|$$
  
Property~\ref{item:piecewise_large_T} holds since by definition of $T^{i+1}$ and by  Property~\ref{item:piecewise_large_T} of the induction hypothesis $$|T^{i+1}|>|T^i|-\regular|T|\geq(1-2\regular i)|T|-\regular|T|\geq (1-2(i+1)\regular)|T|.$$

From the definition of $u_{i+1}$ and the induction hypothesis 
\begin{eqnarray*}
u_{i+1}&=&u_i+p-i\regular \\
&\geq& ip-\regular\sum_{j=1}^ij + p - i\regular \\
&\geq& (i+1)p-\regular\sum_{j=1}^{i+1}j 
\end{eqnarray*}
Property~\ref{item:piecewise_disjoint_window} holds. 
 Before we prove that Properties~\ref{item:piecewise_large_agre} and \ref{item:piecewise_no_end_point} hold we prove the following auxiliary claim.
 
 \begin{claim}\label{clm:piecewise_classifiers_aux_case1}
 For each $h\in T^{i+1}$ and $x\in \cX\cap[u_i,u_{i+1}-\regular]$ it holds that $h(x)=0$.
  \end{claim}
  \begin{proof}
  Assume by contradiction that there is $h\in T^{i+1}$ and $x\in \cX\cap[u_i,u_{i+1}-\regular]$ with $h(x)=1$. 
  This means that there is $k$ with $x\in[a^h_k,a^h_k+p],$ which implies that $a^h_k\leq u_{i+1}-\regular$ and $a_k^h+p\geq u_i.$
From Property~\ref{item:piecewise_no_end_point} there is no end point in the current window $[u_i,u_{i+1}]$; hence $a^h_k+p > u_{i+1}$, which is a contradiction to the definition of $T^{i+1}$ with $I=(u_{i+1}-\regular, u_{i+1})\subseteq[a^h_k,a^h_k+p].$
  \end{proof}

\underline{To prove that Property~\ref{item:piecewise_large_agre} holds for $W_{i+1}$:}  we get that for each $h_1,h_2\in T^{i+1},$ by Claim~\ref{clm:piecewise_classifiers_aux_case1} and the induction hypothesis we have that 
\begin{eqnarray*}
|\{x \in \cX\cap [0,u_{i+1}] | h_1(x)\neq h_2(x)\}| &=& |\{x \in \cX\cap [0,u_{i}] | h_1(x)\neq h_2(x)\}|\\
&+& |\{x \in \cX\cap (u_i,u_{i+1}-\regular] | h_1(x)\neq h_2(x)\}|\\
&+& |\{x \in \cX\cap (u_{i+1}-\regular,u_{i+1}] | h_1(x)\neq h_2(x)\}|\\
&\leq& \regular\sum_{j=1}^{i+2}j|\cX| + \regular \leq \regular\sum_{j=1}^{i+3}j|\cX|
\end{eqnarray*}

\underline{To prove that Property~\ref{item:piecewise_no_end_point} holds for $W_{i+1}$:} if there is an endpoint in $W_{i+1}$ then, since the length of $W_{i+1}$ is smaller than $p-\regular$, its start point is before $u_{i+1}-\regular$. This contradicts Claim~\ref{clm:piecewise_classifiers_aux_case1}.

\item \textbf{Case 2:} there is $I\subseteq W_i$ with $c_1^I\geq\regular|T|$. Since $c_1^I$ increases as we slide $I$, we focus on the first sliding-window $I=(i_1,i_2)\subseteq W_i$ such that $c_1^I\geq\regular|T|$. Since there is no separating $I$ in $W_i$ we get that $c_0^I<\regular|T|.$ 
There are again two cases depending on whether $I$ is at the beginning of $W_i$ or not.

\item \textbf{Case 2.1:} if $i_1=u_i,$ we define 
 $$T^{i+1}=\{h\in T^i| \exists k,  I \cap [a^h_k, a^h_k + p]\neq\emptyset\}\quad\text{and}\quad u_{i+1}=u_i+p+\regular.$$
 
\underline{To prove that Property~\ref{item:piecewise_large_T} holds for $W_{i+1}$:} follows from the induction hypothesis and the fact that  $T^{i+1}=T^{i}\setminus c^I_0$ and $|c^I_0|<\regular|T|.$ 

\underline{To prove that Property~\ref{item:piecewise_disjoint_window} holds for $W_{i+1}$:} follows from the induction hypothesis and the definition of $u_{i+1}.$

 Before we prove that Properties~\ref{item:piecewise_large_agre} and \ref{item:piecewise_no_end_point} hold we need the following auxiliary claim.
  \begin{claim}\label{clm:piecewise_classifiers_aux_case21}
 For each $h\in T^{i+1}$ and $x\in \cX\cap[u_i+\regular,u_i+p-i\regular]$ it holds that $h(x)=1$.
  \end{claim}
\begin{proof}
 For each $h\in T^{i+1}$ there is $k$ such that $(u_i,u_i+\regular)\cap[a^h_k, a^h_k + p]\neq\emptyset$.
 Hence $a^h_k\leq u_i+\regular$ and $a^h_k + p\geq u_i.$
Since there is no end point in the current window $a^h_k+p\notin [u_i,u_i+|W_i|]$ we have that $a_k^h+p > u_i+|W_i| = u_i + p-i\regular.$
To sum up, we have $[u_i+\regular, u_i+p-i\regular]\subseteq [a_k^h,a_k^h+p]$, which proves the claim.


\end{proof}

\underline{To prove that Property~\ref{item:piecewise_large_agre} holds for $W_{i+1}$:} note that $|\{x \in \cX\cap [0,u_{i+1}] | h_1(x)\neq h_2(x)\}|$ is equal to 
\begin{eqnarray*}
|\{x \in \cX\cap [0,u_{i}] | h_1(x)\neq h_2(x)\}| &+& |\{x \in \cX\cap (u_i,u_i+\regular) | h_1(x)\neq h_2(x)\}| \\
&+& |\{x \in \cX\cap [u_i+\regular ,u_i+p-i\regular] | h_1(x)\neq h_2(x)\}|\\
&+& |\{x \in \cX\cap (u_i+p-i\regular,u_i+p+\regular] | h_1(x)\neq h_2(x)\}|\\
&\leq&\regular\sum_{j=1}^{i+2}j,
\end{eqnarray*}
where the inequality follows from the induction hypothesis and Claim~\ref{clm:piecewise_classifiers_aux_case21}.

 \underline{To prove Property~\ref{item:piecewise_no_end_point} holds for $W_{i+1}$:} if there is an $h\in T^{i+1}$ with an end-point in $$[u_{i+1},u_{i+1}+|W_{i+1}|]=[u_{i+1},u_{i+1}+p-(i+1)\regular]$$ then its start-point is in  $$[u_{i+1}-p,u_{i+1}-(i+1)\regular]=[u_{i}+\regular,u_{i}+p-i\regular]\subseteq[u_i,u_i+|W_i|],$$
which is a contradiction to Claim~\ref{clm:piecewise_classifiers_aux_case21} and Property~\ref{item:piecewise_no_end_point} for $W_{i}.$
 
\item \textbf{Case 2.2:} If $i_1>u_i,$ we define 
 $$T^{i+1}=\{h\in T^{i}|\exists k. a^h_k\in I\}\quad\text{and}\quad u_{i+1}=i_2+p.$$

\underline{To prove that Property~\ref{item:piecewise_large_T} holds for $W_{i+1}$:} Since $I$ is the first sliding window with $c_1^I\geq\regular|T|$ there are at most $\regular|T|$ of the hypotheses in $T^i$ that start before $I$. Since $c_0^I<\regular|T|$ there are at most $\regular|T|$ of the hypotheses in $T^i$ that start after $I$. 
In other words  there are at least $1-2\regular$ hypotheses that intersect $I$; i.e.,  $|T^{i+1}|\geq |T^{i}| - 2\regular|T|$. 
By Property~\ref{item:piecewise_large_T} for $W_i$ we have $|T^{i+1}|\geq (1-2(i+1)\regular)|T|.$

\underline{To prove that Property~\ref{item:piecewise_disjoint_window} holds for $W_{i+1}$:} simply note that  $i_2\geq u_i$. 

 Before we prove that Properties~\ref{item:piecewise_large_agre}, \ref{item:piecewise_no_end_point} hold we need the following auxiliary claims.
 \begin{claim}\label{clm:piecewise_classifiers_aux_case2}
 For each $h\in T^{i+1}$ and $x\in \cX\cap[i_2,i_1+p]$ it holds that $h(x)=1$.
  \end{claim}
\begin{proof}
Since for each $h\in T^{i+1}$ it holds that that there is $k$ such that $a^h_k\in I$ then it holds that $i_1\leq a^h_k\leq i_2$. Hence, $[i_2,i_1+p]\subseteq [a^h_k, a^h_k + p].$
\end{proof}
 \begin{claim}\label{clm:piecewise_classifiers_aux_case22}
 For each $h\in T^{i+1}$ and $x\in \cX\cap(u_i,i_1)$ it holds that $h(x)=0$.
 \end{claim}
\begin{proof} 
From Property~\ref{item:piecewise_no_end_point} for $W_i$ we know that there is no end point in the current window $[u_i,u_i+p-i\regular]$ and specifically in $(u_i,i_2)\subseteq[u_i,u_i+p-i\regular]$ (because $I=(i_1,i_2)\subseteq W_i=[u_i,u_i+p-i\regular]$). 
Since for each $h\in T^{i+1}$ it holds that there is $k$ such that $a^h_k\in I=(i_1,i_2)$ then the claim follows. 
\end{proof}
\underline{To prove that Property~\ref{item:piecewise_large_agre} holds for $W_{i+1}$:} take $h_1,h_2\in T^{i+1}$ and note that
\begin{eqnarray*}
|\{x \in \cX\cap [0,u_{i+1}] | h_1(x)\neq h_2(x)\}| &=&|\{x \in \cX\cap [0,u_{i}] | h_1(x)\neq h_2(x)\}| \\
&+& |\{x \in \cX\cap (u_i,i_1) | h_1(x)\neq h_2(x)\}| \\
&+& |\{x \in \cX\cap [i_i,i_2) | h_1(x)\neq h_2(x)\}|\\
&+& |\{x \in \cX\cap [i_2,i_1+p) | h_1(x)\neq h_2(x)\}|\\
&+& |\{x \in \cX\cap [i_1+p,u_{i+1}] | h_1(x)\neq h_2(x)\}|
\end{eqnarray*}
By the induction assumption the first term is at most $\sum_{j=1}^{i+2}j|\cX|$, from Claims~\ref{clm:piecewise_classifiers_aux_case2},\ref{clm:piecewise_classifiers_aux_case22} the second and fourth term are equal to $0$, the third and the fifth terms are at most $2\regular|\cX|$ each because the lengths of $[i_1,i_2]$ and $[i_1+p,i_2+p]$ are of size $\regular. $
This means that we have proven that Property~\ref{item:piecewise_large_agre} holds for window $W_{i+1}.$

\underline{To prove Property~\ref{item:piecewise_no_end_point} holds for $W_{i+1}$:} note that if there is a hypothesis $h\in T^{i+1}$ with an end point in the window $W_{i+1}= [u_{i+1},u_{i+1}+p-(i+1)\regular]$, its corresponding start point is in $[u_{i+1}-p,u_{i+1}-(i+1)\regular]=[i_2,(i_1+\regular) -(i+1)\regular+p]=[i_2,i_1+p-i\regular].$ 
By construction of $T^{i+1}$ there is no $h\in T^{i+1}$ with a start point in the interval $[i_2,i_1+p]$.  
  \end{itemize}
\end{proof}

\section{Time-efficient implementations}\label{apx:time_efficent_implementations}

In this section we apply the general algorithm to a few natural classes. 
We do so by proving that these classes are separable. 
We then show a time-efficient implementations for these classes. 
Interestingly, the memory size used in these implementations are smaller compared to the bounds promised by Theorem~\ref{thm-learning-bounded-memory-appendix}.  
There are a few reasons for that (which we encountered in the threshold functions example). 
First, in Theorem~\ref{thm-learning-bounded-memory-appendix} we used worst-case analysis and assumed that the algorithm must reach $|T|=1.$ 
Sometimes we can stop much earlier if we know that $T\subseteq B_h(\close)$ for some $h$. 
Second, the oracle can use the fact that not all subsets of $\cH$ will be reached as $T$ during the run of the algorithm.  
Third, in case that each $T$ contains large subsets $T_0,T_1\subseteq T$ with $d(S,T_0)=0$ and $d(S,T_1)=1$, an improved implementation of subroutine Estimate (described in the main text),
exists as only one example in $S$ is sufficient in the PAC framework (i.e., the realizable case).

\subsection{Time-efficient implementation to decision lists}

Next we discuss how to apply our general algorithm to the class of decision lists. 
The algorithm goes over $\lceil\log\frac1{\close}\rceil$ levels from top to bottom. 
At each level $i$, as in the proof of Theorem 8,
the set of possible hypotheses are of the form 
 \begin{itemize}
     \item all hypotheses are identical up to level  $j$, for some $j$
     \item all literals in levels $j+1$ to $i-1$ lead to the same value $b$ 
 \end{itemize}
When the algorithm reaches step $i$, all the literals in level $i$ can lead to $\{0,1\}$, see Line~\ref{DL:init_possible_list}.
Then it goes over all literals $\ell,\ell'$ that lead to a different bit in Line~\ref{DL:same_level_split}. 
These two literals define two large subsets of $T$ $$T_0=\{h\in T. \ell \text{ at level } i \text{ leads to } 0\},\quad T_1=\{h\in T. \ell' \text{ at level } i \text{ leads to } 1\}.$$
We remove one of these subsets using Algorithm~\ref{alg:DL_delete} (the constant in the algorithms were chosen arbitrarily). 

Next we test whether we can increase $j$, i.e., whether we know the hypotheses up to level $i-1.$ 
This can happen if there are exactly $i-j-1$ literals possible in levels $j$ up to $i-1.$ 
While $|L_j|>i-j-1$, we test in Algorithm~\ref{alg:DL}, Line~\ref{DL:different_level_split} whether we can eliminate a literal by finding another literal at level $i$ that leads to $\bar{b}$ and apply Algorithm~\ref{alg:DL_delete}.
If there are exactly $i-j-1$ literals at $L_j,\ldots,L_{i-1}$ we can deduce that all hypotheses we consider are similar up until level $i-1$.
In this case we can increase $j$ (Line~\ref{DL:increase_j}) and we know which constraints to add to $C$ to ensure that assignments reaches level $i$ (Line~\ref{DL:add_constraint_to_C}).  
Now the only case we need to consider is that some literal $\ell$ at step $i$ leads to $0$ and its negation $\bar{\ell}$ leads to $1$. 
Unfortunately, we can not set both of them to True and use Algorithm~\ref{alg:DL_delete}. 
Following the proof of Theorem 8
we define $h'$ to be equal to all the rest of the hypotheses in $T$ up until level $j$, at level $i$ literal $\bar{\ell}$ leads to $1$, and in the rest of the levels the remaining literals leads to $0$. From the proof we know that if $h'$ is not $\close$-close to the correct hypotheses then we can delete one of the following hypotheses sets 
$$T_0=\{h\in T. \ell \text{ at level } i \text{ leads to } 0\},\quad T_1=\{h\in T. \bar{\ell} \text{ at level } i \text{ leads to } 1\}.$$

\begin{algorithm} 
{\caption{efficient-DL$(n,\close)$}\label{alg:DL}} 
{
\begin{algorithmic}[1]
\STATE $j:=0, C:=Empty$
\FOR {$i:=1$ to $\lceil\log\frac1\close\rceil$}
\STATE $L_i:=\cup\{x_r,\bar{x}_r\}_r\times\{0,1\}$ \label{DL:init_possible_list}
\WHILE {there are $(\ell,0),(\ell',1)\in L_i$, $\ell'\neq\bar{\ell}$} \label{DL:same_level_split}
\STATE delete-hypotheses$(\ell,0,L_i,\ell',1,L_{i}, C)$
\ENDWHILE
\WHILE {$i>j+1$ and $\exists (\ell,\bar{b})\in L_i$}
\IF{$|L_{j}|>i-j-1$}
\STATE delete-hypotheses$(\ell,\bar{b},L_i,\ell',b,\{L_{j}\ldots,L_{i-1}\}, C)$ \label{DL:different_level_split}
\ELSE 
\STATE $j:=i-1$ \label{DL:increase_j}
\STATE $C := C \cup (x_j = b)\cup\ldots\cup(x_{i-1}=b)$  \label{DL:add_constraint_to_C}
\ENDIF
\ENDWHILE
\IF {$(\ell,0)\in L_i$}
\STATE $b:=0$
\ELSE
\STATE $b:=1$
\ENDIF 
\IF {$(\ell,0),(\bar{\ell},1)\in L_i$}
\STATE $h':=$ as descried in the text 
\IF {Is-close$(h',\close,10^{-4}\close^{-2})$ returns True}
\RETURN $h'$
\ENDIF
\STATE $S=\cX|_{\ell=True, C}$
\IF {Estimate$(S,\close,10^{-4}\close^{-2})<\close$}
\STATE delete $(\ell',1)$ from $\{L_i\}_{i}$
\ELSE
\STATE delete $(\ell,0)$ from $\{L_{i'}\}_{i'}$
\ENDIF
\ENDIF
\ENDFOR
\end{algorithmic}
}
\end{algorithm}

\begin{algorithm} 
{\caption{delete-hypotheses$(\ell,b,\{L_i\}_i,\ell',b',\{L_{i'}\}_{i'}, C)$}\label{alg:DL_delete}} 
{
\begin{algorithmic}[1]
\STATE \textbf{Input: } set of hypotheses with literal $\ell$ leads to bit $b$ in levels $\{L_i\}_i$ and a set of hypotheses with literal $\ell'$ leads to bit $b'$ in levels $\{L_{i'}\}_{i'}$. 
\STATE $C$ - set of constraints all input hypotheses fulfill 
\STATE \textbf{Returns: } delete from one of hypotheses sets 
\STATE $S=\cX|_{\ell=True, \ell'=True, C}$ \label{DL:same_level_delete1}
\IF {Estimate$(S,\close,10^{-4}\close^{-2})<\close$}
\STATE delete $(\ell',1)$ from $\{L_i\}_{i}$
\ELSE
\STATE delete $(\ell,0)$ from $\{L_{i'}\}_{i'}$
\ENDIF \label{DL:same_level_delete2}
\end{algorithmic}
}
\end{algorithm}

\subsection{Time-efficient implementation to equal-piece classifiers}\label{apx:EP_efficient} 
We now consider a time-efficient implementation to equal-piece classifiers defined in the main text.
In the proof of Theorem 7
we moved a sliding-window $I$ of length $\regular.$
For a time-efficient implementation we cannot move a sliding-window in a continuous manner, thus we suggest to move a smaller sliding-window of length $\regular/2$ in a discrete way with jumps of size $\regular/2$. We define at each step $$T_0=\{ h. \forall k,  S \cap [a^h_k, a^h_k + p]\}=\emptyset \quad T_1=\{ h. \exists k,  S \subseteq [a^h_k, a^h_k + p]\}$$

In Line~\ref{alg:EP_init} we initialize our $h$ to be one that always return $0.$
Thus $h$ is empty.
At each step we define $S$ to be our current sliding-window in Line~\ref{alg:EP_define_sliding_window}. In Line~\ref{alg:EP_faster_estimate} we use a faster implementation of  Algorithm Estimate in the main text, 
since in our case $d(S,T_0)=0$ and $d(S,T_1)=1$ and thus a single sample from $S$ suffices. 
If the label is equal to $1$ then we know that we need to add a new value to $h$ in Line~\ref{alg:EP_add_to_correct_hypo} and we know that we can even increase $jump$ by $p.$

\begin{algorithm} 
\label{alg:EP}
{\caption{efficient-EP$(n,\close)$}} 
{
\begin{algorithmic}[1]
\STATE $j := 0$, $h := Empty$ \label{alg:EP_init}
\FOR {$jump:=0,\frac\alpha2,\alpha,\ldots,1$}
\STATE $S=[jump, jump+\alpha/2]$ \label{alg:EP_define_sliding_window}
\REPEAT 
\STATE get labeled example $(x,y)\quad\quad$ 
\UNTIL {$x\in S$} \label{alg:EP_faster_estimate}
\IF {$y=1$} 
\STATE add $jump$ to $h$ \label{alg:EP_add_to_correct_hypo}
\STATE $jump := jump + p$
\ENDIF
\ENDFOR
\RETURN $h$
\end{algorithmic}
}
\end{algorithm}

\subsection{Time-efficient implementation of general bounded-memory algorithm to \texorpdfstring{$\cH_{TH; n}$}{TEXT}}
See Algorithm~\ref{alg:efficient_TH}.
\begin{algorithm} 
{\caption{Efficient Implementation of General Algorithm to $\cH_{TH; n}$}\label{alg:efficient_TH}} 
{
\begin{algorithmic}[1]

\STATE $a_1=0,a_2=1$
\WHILE {$a_2-a_1>\close$}
\STATE $S=\left[a_1+\frac{a_2-a_1}{3},a_2-\frac{a_2-a_1}{3}\right]$
\REPEAT 
\STATE get labeled example $(x,y)\quad\quad$  
\UNTIL {$x\in S$}
\IF {$y=0$} 
\STATE $a_2=a_2-\frac{a_2-a_1}{3}$
\ELSE
\STATE $a_1 = a_1+\frac{a_2-a_1}{3}$
\ENDIF
\ENDWHILE
\RETURN $h_{(a_1+a_2)/2}$
\end{algorithmic}
}
\end{algorithm}

\section{Correctness proof for the general bounded-memory algorithm}\label{apx:technical_proofs: Proofs of Claims from Section 3 in the Main Text}
We start with some technical claims.
\begin{claim}
Let $(A,B,E)$ be a bipartite graph. For any $T\subseteq A$ that is $\regular$-separable there are $S\subseteq B$ with $|S|\geq \regular|B|$, $T_0,T_1 \subseteq T$ with $|T_0|,|T_1| \ge \frac12\regular^2|T|$ and $d_0,d_1 \in \mathbb{R}$ with $d_1-d_0 \geq\frac{\regular}{4}|S|$ such that $h\in T_0$ implies $e(h,S)\leq d_0$ and $h \in T_1$ implies $e(h,S)\geq d_1$.
\end{claim}
\begin{proof}
By the separable property we know that there are $S,T'_0,T'_1$ as in Definition 2 in the main text. 
Sort all $t\in T$ by $e(S,t)$ in an ascending order.
Define $T_1\subseteq T$ as the $\regular|T|$-largest members in $T$ and $T_0\subseteq T$ as the $\regular|T|$ smallest members in $T$. Note that $|d(S,T_1')-d(S,T_0')|\leq|d(S,T_1)-d(S,T_0)|.$
 Assume by way of contradiction that the $(\regular^2/2)|T|$-largest member in $T$, denote it by $t_1$, and the $(\regular^2/2)|T|$-smallest member in $T$, denote it by $t_0$, are too close; i.e., $e(S,t_1)-e(S,t_0)<\regular/2|S|.$ Let us calculate 
\begin{eqnarray*}
\left|d(T_1,S)-d(T_0,S)\right| &=& \frac{e(S,T_1)}{|S||T_1|}-\frac{e(S,T_0)}{|S||T_0|}\\
 &\leq& \frac{\regular^2|T|}{2}\cdot \frac{|S|}{|S||T_1|} + \left(|T_1|-\frac{|T|\regular^2}{2}\right)\frac{e(S,t_1)}{|S||T_1|}  \\
  &-& \left(\frac{\regular^2|T|}{2}\cdot \frac{0}{|S||T_0|} + \left(|T_0|-\frac{|T|\regular^2}2\right)\cdot \frac{e(S,t_0)}{|S||T_0|}\right)\\
(\text{recall: } |T_1|=|T_0|=\regular|T|) &\leq& \frac{\regular}{2} + \left(1-\frac{\regular}{2}\right)\frac{e(S,t_1)}{|S|}-\left(1-\frac{\regular}2\right)\cdot \frac{e(S,t_0)}{|S|}\\
&<&\frac{\regular}{2}+\frac{\regular}{2}=\regular. 
\end{eqnarray*}
Which is a contradiction to the separable property.
\end{proof}

\begin{claim}
There is an algorithm such that for any hypothesis $h$, for any $\close\in(0,1)$ and for any integer $k$ it uses $k$ labeled examples, $O(\log k)$ memory bits,  
and 
\begin{itemize}
\item if $h$ is $\close$-close to the correct hypothesis, then with probability at least $1-2e^{-2k\close^2}$ the algorithm returns True.
\item if $h$ is not $3\close$-close to the correct hypothesis, then with probability at least $1-2e^{-2k\close^2}$ the algorithm returns False.
\end{itemize}
\end{claim}
\begin{proof}
We will show that Algorithm Is-close, main text, has the desired properties. 
Denote by $X_i$ the random variable that is $1$ if the labeled example $(x,y)$ used in the $i$-th step of 
the algorithm 
has $h(x)\neq y$, otherwise $X_i=0.$ 
Denote $\bar{X} = \frac1k\sum X_i.$
Let us consider the two cases presented in the claim 
\begin{itemize}
\item if $h$ is $\close$-close to the correct hypothesis, then $$\E[\bar{X}]\leq \close.$$ 
From Hoeffding's inequality we know that  
$$\Pr[\bar{X} -\E[{\bar{X}}]\geq \close]\leq 2e^{-2k\close^2}.$$
And the last two inequalities imply that $$\Pr[\bar{X} \geq 2\regular]\leq 2e^{-2k\close^2}.$$
Note that line $11$ in the algorithm tests whether $\bar{X}\leq2\close$ or not.  
This means that with probability at least $1-2e^{-2k\close^2}$ the algorithm returns True.
\item if $h$ is not $3\close$-close to the correct hypothesis, then  $$\E[\bar{X}] > 3\close.$$ 
From Hoeffding's inequality we know that  
 $$\Pr[\E[{\bar{X}}] - \bar{X} \geq \close]\leq 2e^{-2k\close^2}.$$
 And the last two inequalities imply that $$\Pr[2\close > \bar{X}]\leq 2e^{-2k\close^2}.$$
This means that with probability at least $1-2e^{-2k\close^2}$ the algorithm returns False.
\end{itemize}
\end{proof}

\begin{claim}
Denote by $f$ the correct hypothesis. 
There is an algorithm such that for any set of examples $S\subseteq\cX$ with $|S|\geq\regular|\cX|$ for any $\tau\in(0,1)$ and for any integer $k$ the algorithm uses $\frac{2k}{\regular}$ labeled examples, $O(\log k + \log\nicefrac1\regular)$ memory bits, 
and returns $\bar{Y}$ with  $|\bar{Y} - d(f,S)]|< \tau,$ 
 with probability at least $1-2(e^{-k\regular}+e^{-2k\tau^2})$.
\end{claim}
\begin{proof}
We will show that Algorithm Estimate, main text, 
has the desired properties. 
Let $m=\frac{2k}{\regular}.$
Denote by $X_i$ the random variable that is $1$ if the labeled example $(x,y)$ used in the $i$-th step of the algorithm has $x\in S$, otherwise $X_i=0.$ 
Denote $\bar{X} = \frac1m\sum X_i$.
Note that $\E[{\bar{X}}]=\frac{|S|}{|\cX|}\geq\regular.$
From Hoeffding's inequality we know that  
$$\Pr[|\bar{X} -\E[{\bar{X}}]|\geq \regular/2]\leq 2e^{-2\frac{2k}{\regular}\cdot\frac{\regular^2}{4}}.$$
In particular, 
$$\Pr[\bar{X}\leq\regular/2 ]\leq 2e^{-k\regular}.$$
This implies that with probability at least $1-2e^{-k\regular}$ there are at least $k$ examples $(x,y)$ with $x\in S.$  
Let us now focus on these $k$ examples. 
Among them, denote by $Y_i$ the random variable that is $1$ if the $i$-th labeled example has $y=1$, otherwise $Y_i=0.$ 
Denote  $\bar{Y} = \frac1k\sum Y_i.$
Note that $\E[{\bar{Y}}]=d(f,S).$
From Hoeffding's inequality we know that  
$$\Pr[|\bar{Y} -\E[{\bar{Y}}]|\geq \tau]\leq 2e^{-2k\tau^2},$$
i.e., with probability at least $1-2e^{-2k\tau^2}$ the algorithm returns an answer $\bar{Y}$ in line $13$ of the algorithm with $|\bar{Y} - d(f,S)]|< \tau.$
\end{proof}

\begin{theorem}\label{thm-learning-bounded-memory-appendix}
For any hypothesis class $\cH$ that is $(\regular,\close)$-separable 
there is a
$$\left(\log|\cH|\cdot\frac{\log\log|\cH|+\log\nicefrac{1}{\regular}}{\regular^5},\, \log|\cH|\cdot\frac{1}{\regular^2},\, 0.1,\,\close\right)-\text{bounded memory algorithm for }\cH.$$

\end{theorem}
\begin{proof}
\textit{Description of the algorithm:} At each iteration there will be a candidate set of hypotheses $T\subseteq\cH$ that contains the correct hypothesis with high probability.
If $T$ is not $\regular$-separable then there is a center $h$ such that $|T\cap B_h(\close)|\geq\regular|T|.$ 
If Algorithm Is-close, main text, 
returns True then $h$ must be $3\regular$-close 
and we are done. 
Otherwise, the correct hypothesis is not in $T\cap B_h(\close)$ and we can move on to the next iteration while removing a large fraction of the hypotheses from $T.$ 

If $T$ is $\regular$-separable then using Claim 4 in main text 
we know that we can remove either $T_0$ or $T_1.$ We then move to the next iteration.


\textit{Number of iterations:} At each iteration we remove at least a fraction of $\regular^2/2$ of the hypotheses in $T$. Thus, the number of iterations the algorithm makes is at most $$s:=\log_{\frac1{1-\regular^2/2}}|\cH|\leq\frac{\log|\cH|}{\alpha^2/2},$$ where the inequality follows from the known fact $1-1/x\leq\ln x$. 

\textit{Number of examples:} At each iteration, the algorithm receives at most $\frac{2k}\regular$ examples. 
Thus, the total number of examples used is at most $$\frac{2k}\regular \cdot s.$$ 

\textit{Number of memory bits:} the memory is composed of two types; one that describes the set of hypotheses $T$ that is currently being examined and $O\left(\log\frac{k}\regular\right)$ bits for all the counters used in the subroutines. We can describe $T$ by the sets that are removed at each iteration, thus we need $s$ bits. 
 Hence, the total number of memory bits is at most $\left(s + \log \frac{k}{\regular} \right)$. 
 
 \textit{Fixing $k$:} We want the error of the probability to be some constant. For that we take $k=O\left(\frac{1}{\regular^2}(\log\log|\cH|+\log\nicefrac{1}{\regular})\right).$ To summarize, the algorithm is  
$$\left(\log|\cH|\cdot\frac{\log\log|\cH|+\log\nicefrac{1}{\regular}}{\regular^5},\, \log|\cH|\cdot\frac{1}{\regular^2},\, 0.1,\,\close\right)-\text{bounded memory algorithm for }\cH.$$

\end{proof}

 \section{The general bounded memory algorithm as a statistical query algorithm}\label{apx:general_alg_sq_algorithm}
The following claims prove that the bounded-memory algorithm can be formalized as a statistical query algorithm.
\begin{claim}
There is an implementation of subroutine \emph{Is-close} that uses one statistical query. 
\end{claim} 
\begin{proof}
To test if hypothesis $h$ is $\close$-close to the correct hypothesis define the query $\psi_h(x,y)=1 \Leftrightarrow h(x)=y.$
We know that
\begin{eqnarray*}
\E_{(x,y)}[\psi_h(x,y)] &=& \frac{1}{|\cX|}\sum_{x\in\cX} I_{h(x) = f(x)},
\end{eqnarray*}
where $I_P$ is the indicator function that is $1$ if and only if $P=True.$
\end{proof}

\begin{claim}
There is an implementation of subroutine \emph{Estimate} that uses one statistical query. 
\end{claim} 
\begin{proof}
Note that 
$$d(S,f) = \frac1{|S|}\sum_{x\in S} I_{f(x) = 1}.$$
Define the query $\psi_S(x,y)=1  \Leftrightarrow x\in S \text{ and } y =1$. 
We know that $\E_{(x,y)}[\psi_S(x,y)] $ is equal to 
$$\frac1{|\cX|} \sum_x I_{x\in S \text{ and } f(x)=1}=\frac1{|\cX|} \sum_{x\in S} I_{f(x)=1}.$$ Thus,
$$\left(\E_{(x,y)}[\psi_S(x,y)]\pm \tau\right)\cdot \frac{|\cX|}{|S|} = d(S,f) \pm \tau\frac{|\cX|}{|S|},$$
which means that we are able to estimate $d(S,f)$ up to error $\tau\frac{|\cX|}{|S|}$ using one statistical query.
\end{proof}



\begin{corollary}
Subroutine $Is-close(h,\close)$ can be simulated in the presence of $\eta$-noise with probability at least $1-\delta$ using $O(\close^{-2}(1-2\eta)^{-2}\ln(\frac{1}{\delta}))$ samples.  Subroutine $Estimate(S,\tau)$ can be simulated using $O(\tau^{-2}(\nicefrac{|S|}{|X|})^{-2}(1-2\eta)^{-2}\ln(\frac{1}{\delta}))$ samples.
\end{corollary}

\end{document}